\documentclass[12pt]{article}

\usepackage{fullpage}
\usepackage{graphicx} 

\usepackage{times,url,bm}
\usepackage[normalem]{ulem}
\usepackage{amsthm,amsmath,amssymb,mathtools,epsfig,color,float,graphicx,verbatim}
\usepackage{algorithm,algorithmic}
\usepackage{bbm}
\usepackage[square,numbers]{natbib}
\usepackage{epstopdf}
\usepackage{tikz}
\usetikzlibrary{shapes.geometric}
\usetikzlibrary{positioning}
\usepackage{pgfplots}
\usepackage{authblk}
\usepackage{comment}
\usepackage{booktabs}
\usepackage{caption}
\usepackage{subcaption}

\usepackage{hyperref}

\hypersetup{
	colorlinks   = true, 
	urlcolor     = blue, 
	linkcolor    = blue, 
	citecolor   = black 
}

\newtheorem{theorem}{Theorem}[section]
\newtheorem{proposition}[theorem]{Proposition}
\newtheorem{lemma}[theorem]{Lemma}
\newtheorem{corollary}[theorem]{Corollary}
\newtheorem{definition}[theorem]{Definition}

\newtheorem{assumption}[theorem]{Assumption}

\newcommand{\reals}{\mathbb{R}}
\newcommand{\E}{\mathbb{E}}

\newcommand{\sign}{\mathrm{sign}}

\newcommand{\argmax}[1]{\underset{#1}{\mathrm{argmax}}}

\newcommand{\relu}[1]{\left[ #1 \right]_+}
\newcommand{\set}[1]{\left\{#1\right\}}

\newcommand{\ba}{\mathbf{a}}
\newcommand{\be}{\mathbf{e}}
\newcommand{\bx}{\mathbf{x}}
\newcommand{\bX}{\mathbf{X}}
\newcommand{\bw}{\mathbf{w}}

\newcommand{\bb}{\mathbf{b}}
\newcommand{\bu}{\mathbf{u}}
\newcommand{\bv}{\mathbf{v}}
\newcommand{\bz}{\mathbf{z}}

\newcommand{\by}{\mathbf{y}}
\newcommand{\bY}{\mathbf{Y}}

\newcommand{\bxi}{\boldsymbol{\xi}}

\newcommand{\bzero}{\mathbf{0}}

\newcommand{\Ocal}{\mathcal{O}}

\newcommand{\Dcal}{\mathcal{D}}

\newcommand{\Ncal}{\mathcal{N}}

\newcommand{\Pcal}{\mathcal{P}}

\newcommand{\Ucal}{\mathcal{U}}

\newcommand{\norm}[1]{\left\|#1\right\|}
\newcommand{\inner}[1]{\left\langle#1\right\rangle}
\newcommand{\p}[1]{\left(#1\right)}
\newcommand{\pcc}[1]{\left[#1\right]}
\newcommand{\abs}[1]{\left|#1\right|}
\newcommand{\ceil}[1]{\left\lceil#1\right\rceil}

\newcommand{\polylog}{\text{polylog}}
\newcommand{\one}[1]{\mathbbm{1}\left\{#1\right\}}
\newcommand{\pr}{\mathbb{P}}

\DeclareMathOperator{\daw}{daw}
\DeclareMathOperator{\interior}{int}

\newcommand{\secref}[1]{Sec.~\ref{#1}}

\newcommand{\figref}[1]{Fig.~\ref{#1}}
\renewcommand{\eqref}[1]{Eq.~(\ref{#1})}
\newcommand{\lemref}[1]{Lemma~\ref{#1}}
\newcommand{\corollaryref}[1]{Corollary~\ref{#1}}
\newcommand{\thmref}[1]{Thm.~\ref{#1}}
\newcommand{\propref}[1]{Proposition~\ref{#1}}
\newcommand{\appref}[1]{Appendix~\ref{#1}}

\newcommand{\itemref}[1]{Item~\ref{#1}}
\newcommand{\asmref}[1]{Assumption~\ref{#1}}

\title{How Many Neurons Does it Take to Approximate the Maximum?}

\author[1]{Itay Safran}
\author[2]{Daniel Reichman}
\author[1]{Paul Valiant}
\affil[1]{Purdue University}
\affil[2]{Worcester Polytechnic Institute}

\date{}

\begin{document}

\maketitle

\begin{abstract}
We study the size of a neural network needed to approximate the maximum function over $d$ inputs, in the most basic setting of approximating with respect to the $L_2$ norm, for continuous distributions, for a network that uses ReLU activations. We provide new lower and upper bounds on the width required for approximation across various depths. Our results establish new depth separations between depth 2 and 3, and depth 3 and 5 networks, as well as providing a depth $\mathcal{O}(\log(\log(d)))$ and width $\mathcal{O}(d)$ construction which approximates the maximum function. Our depth separation results are facilitated by a new lower bound for depth 2 networks approximating the maximum function over the uniform distribution, assuming an exponential upper bound on the size of the weights. Furthermore, we are able to use this depth 2 lower bound to provide tight bounds on the number of neurons needed to approximate the maximum by a depth 3 network. Our lower bounds are of potentially broad interest as they apply to the widely studied and used \emph{max} function, in contrast to many previous results that base their bounds on specially constructed or pathological functions and distributions.  
\end{abstract}

\section{Introduction}

How and why \emph{depth} helps neural networks to excel in applications is one of the central challenges in the quest to understand deep learning. Both classical circuit complexity and modern deep learning theory is guided by the intuition that a modest increase in depth often leads to drastic---and often exponential---improvements in the expressive power of the circuit or network described, along with concomitant improvements in key measures of performance, including efficiency as measured by width or neuron count, and  approximation accuracy. Despite this firm intuition, and much recent encouraging evidence of the \emph{practical} expressive power of deep networks (and hence also deep circuits), theoretical insight to illuminate these phenomena remains scarce; and each additional layer of depth adds often prohibitive new theoretical challenges. This paper contributes to this area by providing several new lower and upper bounds across a range of depths for a \emph{natural} function, with respect to a  meaningful notion of approximation, in a realistic circuit/network model with realistic parameters.

While it is known that any continuous function can be approximated arbitrarily well by a shallow (depth 2) network \citep{cybenko1989approximation,leshno1993multilayer}, these constructions of depth 2 approximations typically require exponential size, as a function of the input dimension in the worst case. This raises the natural question: when and how can increasing the depth beyond 2 allow us to drastically improve the width, so as to be polynomial in the input dimension.
Indeed, significant study has been directed towards understanding the benefit of depth in expressing functions. For example, there are several constructions of functions which can be approximated by a depth 3 and polynomial width network, but require exponentially many neurons to be approximated by a network of depth 2 \citep{eldan2016power,daniely2017depth,safran2017depth,safran2019depth,venturi2021depth,hsu2021approximation,safran2022optimization}. 
It is also known that highly oscillatory or even sufficiently non-linear functions may require exponentially fewer neurons when approximated by networks whose depth scales with the problem's parameters (e.g.\ the target accuracy) than by constant depth networks \citep{telgarsky2016benefits,yarotsky2017error,liang2017deep,safran2017depth}. While such results establish the expressive superiority of depth over width, the functions used to demonstrate this are at times somewhat pathological, the configuration of the network which approximates them efficiently is highly stylized, or the assumed data distribution is quite complex. Therefore we are motivated to study the effect of depth on efficiency for ``natural", well-studied functions that commonly arise in machine learning tasks. 

    We study the effect of depth on the quality of approximation for the maximum function 
    \[
        f_d(x_1,\ldots,x_d)\coloneqq\max\{x_1,\ldots,x_d\},
    \]
    where $(x_1, \ldots, x_d) \in \mathbb{R}^d$. 
    This function enjoys many favorable properties: It is a fundamental mathematical function; it has a very simple structure; it is easy to compute in linear time (assuming comparisons between real numbers can be done in time $\Ocal(1)$), and it is used explicitly in popular neural network architectures (e.g.\ max pooling\ \citep{krizhevsky2017imagenet,simonyan2014very,ioffe2015batch}). 
    Computing the maximum is important in reinforcement learning tasks (choosing an action maximizing the expected reward) and has received attention in theoretical neuroscience~\citep{kriener2020robust}.
    Several works have studied how to compute the maximum efficiently using a neural network~\citep{arora2016understanding,hertrich2021towards,matoba2022theoretical,haase2023lower}. However, most known results only deal with $L_{\infty}$ approximations or even exact computation of this function, which is a far more stringent notion of approximation than the $L_2$ approximation which is often the metric chosen for  machine learning applications, and is the metric we study in this paper (see related work subsection and \secref{sec:prelims} for further discussion). 
    
    We provide new lower and upper bounds for approximating the maximum function with respect to a broad class of distributions including the uniform distribution over the unit hypercube, and the Gaussian distribution (see \asmref{asm:dist}). We show that for any natural $k\le\Ocal(\log(\log(d)))$,\footnote{Unless stated otherwise all logarithms are base $2$.} ReLU networks of depth $2k+1$ and width $\mathcal{O}\left(d^{1+1/(2^k-1)}\right)$  can approximate the maximum function to arbitrary accuracy given sufficiently large weights. In particular, this implies a depth 3 and width $\mathcal{O}(d^2)$ approximation; a depth 5 and width $\mathcal{O}(d^{4/3})$ approximation; and a depth $\mathcal{O}(\log(\log(d)))$ and width $\mathcal{O}(d)$ approximation. Moreover, assuming an exponential upper bound on the size of the weights, we show corresponding lower bounds for approximating the maximum using ReLU networks. Specifically, we show a depth 2 lower bound requiring width $d^{\ell}$ for any $\ell\ge1$ for obtaining a target accuracy of $d^{-c\cdot\ell}$ for constant $c>0$, and a depth 3 lower bound requiring width $\Omega(d^2)$, establishing the tightness of our depth 3 upper bound.
    
    We remark that by performing a change of variables in our lower bounds, one can rescale the domain of approximation as desired at the cost of rescaling the target accuracy (see \lemref{lem:rescaling} for a formal statement). For example, rescaling the domain to $[0,d]$ where the maximum function has constant variance will multiply the target accuracy by a factor of $d^2$ in all of our lower bounds. Moreover, by further rescaling, our lower bounds show that the maximum function cannot be approximated to better than constant accuracy if the domain of approximation scales polynomially with the input dimension $d$. Due to the fact that our upper bounds on the required width are independent of the target accuracy (better accuracy is obtained by increasing the magnitude of the weights of the approximating network), scaling the domain of approximation does not affect the width requirement in our upper bounds. Therefore, our results also establish several new depth-based separations for approximating the maximum function to better than constant accuracy. See Table~\ref{tbl} for a comparison of our bounds.

    \begin{table}
        \caption{New results in this paper for approximating the maximum function using ReLU neural networks uniformly over the domain $[0,1]^d$. We provide a polynomial separation between depth 2 and 3, and for the same target function, we provide a polynomial separation between depth 3 and 5, requiring widths of $\Omega(d^2)$ and $\Ocal(d^{4/3})$, respectively. We also provide a novel upper bound which only requires a perhaps surprising network depth of $\Ocal(\log(\log(d)))$ for approximating the maximum using linear width. This in contrast to the previously best known approximation which requires width and overall size linear in $d$ but depth $\Ocal(\log(d))$ to compute the maximum exactly \citep{arora2016understanding,matoba2022theoretical}. The theta notation in the target accuracy column hides an absolute constant, and our depth 2 lower bound holds for any value of the parameter $\ell\ge1$. An asterisk $^{(*)}$ denotes that, beyond the ReLU activation, the bound applies to any activation satisfying a mild assumption (\asmref{asm:dist}); and a dagger $^{(\dag)}$ denotes that an exponential upper bound on the magnitude of the weights is assumed.}
        \label{tbl}
        \centering
        \begin{tabular}{c|cc|c}
            \toprule
            Depth & Target accuracy & Width lower bound & Width upper bound\\
            & & &(any accuracy)\\
            \midrule
            2 & $d^{-\theta(\ell)}$, any $\ell\geq 1$& $\Omega(d^{\ell})$~~(\thmref{thm:depth2})~$^{(*,\dag)}$ &\\
            3 & $d^{-\theta(1)}$ & $\Omega(d^2)$~~(\thmref{thm:depth3_lb})~$^{(\dag)}$ & $\Ocal(d^2)$~~(\thmref{thm:depth3_approx})\\
            5 & $d^{-4.5}$ & $d$~~(\thmref{thm:size_d_lb})~$^{(*)}$ & $\Ocal(d^{4/3})$~~(\thmref{thm:deep_approx})\\
            $\Ocal(\log(\log(d)))$ & $d^{-4.5}$ & $d$~~(\thmref{thm:size_d_lb})~$^{(*)}$ & $\Ocal(d)$~~(\thmref{thm:deep_approx}) \\
            \bottomrule
        \end{tabular}
    \end{table}

    It is interesting to compare lower bounds for continuous neural networks over continuous domains to lower bounds for discrete neural networks such as threshold circuits over $\{0,1\}^d$. Devising superlinear lower bounds for depth three threshold circuits (with polynomial weights on the output gate) has been obtained relatively recently after decades of research for a family of complicated functions~\cite{kane2016super} which cannot be computed by circuits with $o(d^{3/2}/\log ^3 (d))$ threshold gates. Our tight quadratic lower bound for depth $3$ networks with ReLU gates approximating the maximum function suggests that proving lower bounds for the continuous case is a more amenable task and that further superlinear lower bounds for bounded depth networks might be achievable over the continuous domain.
    
    The remainder of this paper is structured as follows: After presenting our contributions in this paper in more detail below, we discuss related work in the literature. In \secref{sec:prelims} we present the notation used throughout this paper. \secref{sec:positive} details our positive approximation results (upper bounds) and \secref{sec:negative} details our negative inapproximability results (lower bounds). Lastly, \secref{sec:summary} summarizes and discusses potential future work directions.
    
    \subsection*{Our contributions}

    \begin{itemize}
        \item
        We exhibit a construction to approximate the maximum function arbitrarily well---in the $L_2$ sense---using a depth 3 and width $d(d+1)$ ReLU network (\thmref{thm:depth3_approx}). Interestingly, to increase the accuracy of this construction, we increase the size of the weights, but do not need to change the network architecture. This construction arises from a piecewise-linear decomposition of the maximum function.

        \item 
        We ``compose" the above construction with itself, so as to enable different depth-width tradeoffs. This reinterpretation of \thmref{thm:depth3_approx} provides new upper bounds for expressing the maximum function across odd depths, with the required width dropping rapidly as larger depths are used (\thmref{thm:deep_approx}).
        

        \item
        By contrast, we show that these constructions at depths 3 and higher are impossible at depth 2: the width of a depth 2 network approximating the maximum function \emph{must} depend on the desired approximation accuracy. This thus shows a polynomial separation between depths 2 and 3 for approximating the maximum function. Our analysis relates to the seminal technique developed in \citet{eldan2016power}, analyzing the Fourier spectrum of the maximum function, to show that, assuming exponential upper bounds on the size of the weights, we show a lower bound for approximating the maximum function on a compact domain (\thmref{thm:depth2}). 

        \item 
        We show a polynomial separation between depth 3 and depth 5 neural networks.
        Using a combinatorial argument, we show that depth 3 ReLU networks with first hidden layer of width at most $d^2/5$ cannot capture the full structure of the maximum function on the hypercube, reducing the approximation error lower bound to the accuracy achieved by the second hidden layer. Using our previous lower bound in \thmref{thm:depth2}, this implies an approximation lower bound for depth 3 ReLU networks (\thmref{thm:depth3_lb}). Together with \thmref{thm:depth3_approx}, this establishes a tight bound of $\Theta(d^2)$ for approximation of the maximum using depth 3; and when combined with our \thmref{thm:deep_approx} which implies a depth 5 and width $\Ocal(d^{4/3})$ approximation, this provides a polynomial separation between depth 3 and 5.

        \item 
        Lastly, we observe that any neural network  (regardless of depth or activation function) approximating the maximum must have at least $d$ neurons in its first hidden layer. 
        Thus, known upper bounds on the number of ReLUs needed to compute the maximum precisely which require 
        size $\Ocal(d)$ (e.g.\ \citet{arora2016understanding,matoba2022theoretical}) are optimal up to a constant factor.


    \end{itemize}
    
    \subsection*{Related work}

    \paragraph{Exact computation and approximation of the maximum function}

    Quite a few recent works have studied the problem of exact computation of the maximum function using ReLU neural networks. 
    \citet{arora2016understanding} establish that any $d$-dimensional, piecewise-linear function can be expressed exactly using a depth $\lceil\log(d+1)\rceil$ ReLU network. Observe that we are able to obtain a more depth efficient implementation of the maximum function, although in contrast to \citet{arora2016understanding} we construct a network \emph{approximating} the function rather than computing it exactly, and our overall size requirement is slightly larger. The construction of \citet{arora2016understanding} implies that $\max\{x_1,x_2\}$ is computable by depth 2 ReLU networks, 
    and $\max\{x_1,\ldots,x_4\}$ is computable by depth 3 ReLU networks. In contrast, \citet{mukherjee2017lower} show that the function $(x_1,x_2)\mapsto\max\{0,x_1,x_2\}$ (which can be shown to be equivalent to computing the maximum over three inputs) cannot be computed exactly by a depth 2 ReLU network regardless of its width. \citet{hertrich2021towards} conjecture an analogous impossibility result for computing $\max\{0,x_1,\ldots,x_4\}$ exactly using a depth 3 network, and partially resolve it by assuming a certain restriction on the structure of the computing network. \citet{haase2023lower} further show the uncomputability of $\max\{0,x_1,\ldots,x_4\}$ assuming the computing network has depth 3 and integral weights, which also implies that the exact computation of $f_d(\cdot)$ using integer weights requires depth $\Omega(d)$. Since our approximation of $f_d(\cdot)$ can be done using integer weights, this shows a depth separation between exact and approximate computation of $f_d(\cdot)$ using integer weights, as the former requires depth $\Omega(\log(d))$ whereas for the latter depth $\Ocal(\log(\log(d)))$ suffices.

    Despite these efforts, this conjecture is still open. Furthermore, it is currently unknown whether there exists \emph{any} piecewise-linear function which cannot be computed exactly using depth 3 ReLU networks. We stress that the aforementioned lower bounds are concerned with \emph{exact} computation, whereas our notion of approximation is markedly different since we consider lower bounds with respect to the $L_2$ norm rather than requiring zero $L_{\infty}$ loss. This is a less stringent approximation requirement in the sense that an $L_2$ lower bound implies an $L_{\infty}$ lower bound, but not vice versa.

    
    In contrast to the exact computation requirement discussed above, \citet{matoba2022theoretical} consider approximations of the maximum function with respect to the $L_{\infty}$ norm. They show a family of networks that increase in accuracy as the size of the approximating network increases, whereas we provide a construction of a network achieving arbitrarily good accuracy with fixed width (by increasing the size of the weights). Additionally, under the restriction of the approximating network to have a certain symmetry structure, they also show approximation lower bounds for the maximum function, 
    however these do not hold in general if we relax the symmetry assumption, whereas our lower bounds hold only under the (mild) assumption that the approximating network has exponentially bounded weights.
    
    \paragraph{Separations between depth 2 and 3}
    The seminal work of \citet{eldan2016power} was the first to establish the existence of a (continuous) function that can be approximated efficiently using networks of depth 3, whereas any network of depth 2 would require width exponential in the input dimension to achieve better than constant accuracy. Later, \citet{daniely2017depth} showed a separation using a different technique which applies to a compactly supported distribution, but requires an exponential upper bound on the magnitude of the weights of the approximating depth 2 network. Following these works, additional separation results between depth 2 and 3 were shown (e.g.\ \citep{venturi2021depth,hsu2021approximation}), including reductions to the results of \citet{eldan2016power} and \citet{daniely2017depth} that however hold for much simpler functions than the ones originally used (e.g.\ \citep{safran2017depth,safran2019depth,safran2022optimization}). Nevertheless, due to the reduction proof technique used, these results inevitably inherit the arguably more complicated distributions over the data used in \citet{eldan2016power} and \citet{daniely2017depth}. 
    In contrast, our separation between depth 2 and 3 holds for both the simple maximum function and for the uniform distribution over a hypercube, albeit providing a polynomial separation and requiring an exponential upper bound on the magnitude of the weights.

    \paragraph{Limitations of deeper architectures}

   Moving beyond depth $2$, there are known constructions of functions that can be approximated by a small sized network (with no restriction on its depth), whereas an approximation to similar accuracy using constant depth networks may require exponentially many more neurons.
    Such lower bounds, however, are based on two main arguments and suffer from certain drawbacks making them incomparable to our results. We discuss these lower bounds and their limitations in more detail below.
    
    \subparagraph{Region-counting-based depth separations}
    The seminal work of \citet{telgarsky2016benefits} first established depth separations between deep architectures. It is shown that a deep ReLU network can realize a one-dimensional, rapidly oscillating sawtooth function, whereas a shallower architecture cannot generate sufficiently many linear segments to be able to approximate this function efficiently.
    If one wishes, however, to learn this efficient representation of this function using the deeper architecture, then it is known that this cannot be done efficiently using standard techniques such as the gradient descent algorithm \citep{malach2021connection}. Different lower bounds exist that build on this region counting proof technique, but focus on smooth and non-linear target functions that may be more prone to be learned efficiently using gradient descent \citep{yarotsky2017error,liang2017deep,safran2017depth}. Nevertheless, there's some theoretical work which shows that when initializing deep ReLU networks, the expected number of linear regions our initialization will have is merely linear in the size of the network, and further empirical evidence suggests that this number does not tend to increase significantly, indicating that depth will impart no practical benefit for approximating these target functions \citep{hanin2019complexity,vardi2021size}. On the other hand, our results which focus on the maximum function and do not rely on region counting, still leave open the possibility of an optimization-based result to be shown which will demonstrate this separation in a more practical setting. 

    
    \subparagraph{Size lower bounds and connections to circuit complexity}

    A different, less direct approach for showing approximation lower bounds for neural networks relies on the connection between threshold circuits and neural networks. \citet{mukherjee2017lower} derive sub-linear size lower bounds for neural networks by showing reductions to known threshold circuit lower bounds. \citet{vardi2021size} use communication complexity to provide linear size lower bounds for approximating a smoothed version of the binary IP mod 2 function. These results provide a different result than ours since they imply a linear width lower bound for approximating various functions using depth 3 networks, while we show a quadratic lower bound for depth 3 ReLU networks. Moreover, the lower bounds in the aforementioned papers are for the size of the network, rather than the required width for some given depth. For this reason such results cannot establish the superiority of depth over width, since these two quantities can be traded off evenly in such lower bounds, whereas our results establish a non-symmetric trade-off which indicates that depth is more efficient for approximating the maximum function.
    

    
    \section{Preliminaries and notation}\label{sec:prelims}

    \paragraph{Notation and terminology}
    We let $[n]$ be shorthand for the set $\{1,\ldots,n\}$. We denote vectors using bold-faced letters (e.g.\ $\bx$) and matrices or random variables using upper-case letters (e.g.\ $X$). Multivariate random variables are denoted using bold-faced upper-case letters (e.g.\ $\bX$). Given a vector $\bx=(x_1,\ldots,x_d)$, we let $\norm{\bx}_p$ denote its $\ell_p$ norm which is given by $\p{\sum_{i=1}^d|x_i|^p}^{1/p}$, where the case $p=\infty$ is defined as $\norm{\bx}_{\infty}=\max_{i\in[d]}|x_i|$. Throughout, we use the notation $f_d(\bx)\coloneqq\max\{x_1,\ldots,x_d\}$ for the maximum function, $\relu{x}=\max\{0,x\}$ for the ReLU activation function, and for any natural $k\ge1$ we use the notation $\beta(k)\coloneqq\frac{1}{2^k-1}$. A function $f:D\to\reals$ defined in some domain $D\subseteq\reals^d$ is \emph{(continuous) piecewise-linear} if there exists a finite partition $D=\cup_{i}D_i$ such that $f$ is linear on $D_i$ for all $i$, where each $D_i$ is a closed set which is referred to as a \emph{linear region} of $f$. We let $\Ucal(A)$ denote the uniform distribution on a set $A\subseteq\reals^d$.

    \paragraph{Neural networks}

    We consider fully connected, feed-forward neural networks, computing functions from $\reals^d$ to $\reals$. A $\sigma$ neural network consists of layers of neurons. In every layer except for the output neuron, an affine function of the inputs is computed, followed by a computation of the non-linear activation function $\sigma:\reals\to\reals$. The single output neuron simply computes an affine transformation of its inputs. Each layer with a non-linear activation is called a \emph{hidden layer}, and the \emph{depth} of a network is defined as the number of hidden layers plus one. The \emph{width} of a network is defined as the number of neurons in the largest hidden layer which we generally denote by $k$, and the \emph{size} of the network is the total number of neurons across all layers.

    \paragraph{Approximation error}
    Since we consider a regression setting in which a neural network $\Ncal:\reals^d\to\reals$ computes a real function of its input, we define our approximation error with respect to an underlying distribution $\Dcal$ on $\reals^d$ and we consider the square loss. Formally, given a predictor $\Ncal$, a target function $f:\reals^d\to\reals$ and an underlying distribution $\Dcal$, our approximation error is defined as
    \[
        \E_{\bx\sim\Dcal}\pcc{\p{\Ncal(\bx) - f(\bx)}^2}.
    \]
    In words, the $L_2$ approximation defined above corresponds to the expected square error when sampling an instance from the underlying distribution $\Dcal$, labelling it using the target function $f$ we are trying to approximate, and making a prediction using a given neural network hypothesis $\Ncal$. This makes this notion of approximation a natural choice for showing approximation lower and upper bounds, as a lower bound for certain class of neural network predictors implies the existence of a particular learning problem where the architecture being considered is unable to achieve population loss better than a certain quantity, whereas a network $\Ncal$ which achieves small loss implies that this class of networks can express a good predictor.

    \section{Deep ReLU approximations}\label{sec:positive}

    In this section, we focus on positive approximation results for the maximum function. We now begin with stating our assumptions on the underlying distribution generating the data, but first we would need the following definition.

    \begin{definition}
        Given some $\delta>0$, we say that a vector $\bx=(x_1,\ldots,x_d)$ is \emph{$\delta$-separated} if for all $i\neq j$, and $x_j\neq0$ we have that
        \[
            \frac{x_i}{x_j}\notin[1-\delta,1+\delta].
        \]
    \end{definition}
    We denote the set of $\delta$-separated vectors in $d$-dimensional space by
    \[
        S_{\delta}\coloneqq\set{\bx\in\reals^d:\bx\text{ is }\delta\text{-separated}}.
    \]
    The above definition essentially guarantees that each pair of coordinates in $\bx$ have a ratio whose distance from $1$ is at least $\delta$ for some real $\delta>0$. Since our construction is sensitive to instances where there are coordinates that are extremely close, we would need to make sure that points that violate $\delta$-separateness are sufficiently scarce. To this end, we make the following assumption on the distribution of the data.
    


    \begin{assumption}\label{asm:dist}
        The distribution $\Dcal$ satisfies the following:
        \begin{enumerate}
            \item\label{item:dist1}
            \[
                \E_{\bX\sim\Dcal}\pcc{\norm{\bX}_{\infty}^2}<\infty.
            \]
            \item\label{item:dist2}
            \[
                \lim_{\delta\to0}\pr_{\bX\sim\Dcal}\pcc{\bX\notin S_{\delta}}=0.
            \]
        \end{enumerate}
    \end{assumption}
    \itemref{item:dist1} merely requires that the tail of $\bX$ is sufficiently well-behaved in the sense of having a finite second moment for its infinity norm, and \itemref{item:dist2} requires that it becomes increasingly unlikely to draw an instance from $\Dcal$ which isn't $\delta$-separated as $\delta$ decreases. These hold, for example, when the coordinates of $\bX$ are i.i.d.\ and follow any absolutely continuous distribution with a bounded density and a finite second moment, or when a certain continuous noise is added to a sufficiently concentrated random variable $\bX$ (see \appref{app:delta_sep_dists} for formal examples).

    We now turn to formally state our positive approximation result for approximating the maximum function using depth 3 ReLU neural networks.

    \begin{theorem}\label{thm:depth3_approx}
        Let $\Dcal$ be any distribution satisfying \asmref{asm:dist}. Then for any $\varepsilon>0$ and natural $d\ge2$, there exists a ReLU network $\Ncal$ of depth $3$ and width $d(d+1)$ such that
        \[
            \E_{\bx\sim\Dcal}\pcc{\p{\Ncal(\bx) - f_d(\bx)}^2} \le \varepsilon.
        \]
    \end{theorem}

    The proof of the above theorem, which appears in \appref{app:positive_proofs}, relies on the observation that the structure of the maximum function is such that its surface consists of $d$ linear regions (corresponding to the subsets of $\reals^d$ where each coordinate is maximal). Since each such region has exactly $d-1$ faces (corresponding to the hyperplanes where one coordinate overtakes the other), we can use the first hidden layer to compute the linear function $\bx\mapsto x_i$ and ``peel off'' the surface of the function at the relevant faces using a ReLU neuron with a very large negative slope. We then use the second hidden layer to truncate negative values. By adding such ``polytope functions'', we are able to obtain a good approximation of the maximum function at points where the coordinates are sufficiently distant from each other. We refer the reader to Definition~\ref{def:depth3} for the formal construction and \figref{fig:depth3} for an illustration.

    \begin{figure}[]
        \centering
        \begin{subfigure}[b]{0.3\textwidth}
            \centering
            \begin{tikzpicture}[scale=0.55]
                \begin{axis}[zmin=-2,zmax=2,view={300}{40},grid=both,restrict z to domain*=-2:1]
                    \addplot3 [surf,unbounded coords=jump,shader=interp,domain=-1:1,y domain=-1:1,samples=80] {13.75*x>13.75*y-2.4?max(0,x)-max(0,10*y-10*x) - max(0,-x)-max(0,10*y-10*x):NaN};
                \end{axis}
            \end{tikzpicture}
            \caption{}
            \label{fig:a}
        \end{subfigure}
        \hfill
        \begin{subfigure}[b]{0.3\textwidth}
            \centering
            \begin{tikzpicture}[scale=0.55]
                \begin{axis}[zmin=-2,zmax=2,view={300}{40},grid=both]
                    \addplot3 [surf,shader=interp,domain=-1:1,y domain=-1:1,samples=80] {max(0,max(0,x)-max(0,10*y-10*x)) - max(0,max(0,-x)-max(0,10*y-10*x))};
                \end{axis}
            \end{tikzpicture} 
            \caption{}
            \label{fig:b}
        \end{subfigure}
        \hfill
        \begin{subfigure}[b]{0.3\textwidth}
            \centering
            \begin{tikzpicture}[scale=0.55]
                \begin{axis}[view={300}{40},grid=both]
                    \addplot3 [surf,shader=interp,domain=-1:1,y domain=-1:1,samples=80] {max(0,max(0,x)-max(0,10*y-10*x)) - max(0,max(0,-x)-max(0,10*y-10*x))  +  max(0,max(0,y)-max(0,10*x-10*y)) - max(0,max(0,-y)-max(0,10*x-10*y))};
                \end{axis}
            \end{tikzpicture}
            \caption{}
            \label{fig:c}
        \end{subfigure}
        \caption{Three-step polytope approximation of $(x,y)\mapsto\max\{x,y\}$. Subfigure~\ref{fig:a} plots the depth 2 network $\relu{x}-\relu{10y-10x}-\relu{-x}-\relu{10y-10x}$ which computes $\max\{x,y\}$ on the polytope $\set{(x,y)\in[-1,1]^2:x\ge y}$. In Subfigure~\ref{fig:b}, a second layer of ReLUs is utilized to clip negative values that are outside of the polytope to zero, plotting the depth 3 network $\Ncal(x,y)\coloneqq\relu{\relu{x}-\relu{10y-10x}}-\relu{\relu{-x}-\relu{10y-10x}}$. Lastly, Subfigure~\ref{fig:c} plots the depth 3 network $\Ncal(x,y)+\Ncal(y,x)$ which is an effective approximation of $\max\{x,y\}$. We remark that while $\max\{x,y\}$ can be computed exactly using depth 2 ReLU networks, the figure is intended for illustration purposes of our construction idea used in \thmref{thm:depth3_approx}, which generalizes to any input dimension $d$. Best viewed in color.}
        \label{fig:depth3}
    \end{figure}



    It is interesting to note that the width of the approximating network in our result only scales with the input dimension $d$, and does not scale with the desired target accuracy. Rather, by increasing the magnitude of the weights of the approximating network we can control the accuracy of the approximation. This is in contrast to many other approximation regimes where an improvement in the approximation accuracy necessitates an increase in width. E.g., when approximating the maximum function using depth 2 (see \propref{prop:finite_d} in the appendix) or when approximating non-linear functions using ReLU networks (see \citet{safran2017depth}).
    
    Our result allows the approximation of the maximum function using a network of size $\Ocal(d^2)$. It is known, however, that the maximum function can be computed exactly using a smaller network of size $\Ocal(d)$ if we allow depth $\Ocal(\log(d))$. It is therefore natural to ask whether by utilizing depth, we can obtain more efficient approximations of the maximum function using our construction. Perhaps surprisingly, we are able to approximate the maximum function using linear width but by only requiring the depth to scale as $\Ocal(\log(\log(d)))$. More formally, we have the following theorem.
    \begin{theorem}\label{thm:deep_approx}
        Let $\Dcal$ be any distribution satisfying \asmref{asm:dist}. Then for any $\varepsilon>0$ and naturals $d\ge58$ and $1\le k\le\lceil\log(\log(d)+1)\rceil$, there exists a ReLU network $\Ncal$ of depth $2k+1$ and width at most $20d^{1+\frac{1}{2^k-1}}$ such that
        \[
            \E_{\bx\sim\Dcal}\pcc{\p{\Ncal(\bx) - f_d(\bx)}^2} \le \varepsilon.
        \]
        In particular, we have a ReLU network of width $40d$ and depth $2\lceil\log(\log(d)+1)\rceil+1$ which approximates $f_d(\bx)$ to accuracy $\varepsilon$ with respect to the distribution $\Dcal$.
    \end{theorem}

    The proof of the above theorem, which appears in \appref{app:positive_proofs}, exploits the key observation that the maximum taken over sub-vectors of maxima is the maximum of the vector. We can thus partition our input into smaller batches and use \thmref{thm:depth3_approx} to compute the maximum over each of these batches. Since we may vary the size of the batches across layers, we can gradually take larger and larger batches as our computation propagates deeper in the network, while keeping the width roughly the same across all layers. This enables a double exponential decay in the number of maxima computed at each layer, requiring depth of only $\Ocal(\log(\log(d)))$ for approximating the maximum using width linear in $d$ (see Definition~\ref{def:deep} for the formal construction and \figref{fig:deep} for an illustration). While our result is not directly comparable to previous approximations of the maximum function \citep{arora2016understanding,matoba2022theoretical} since these allow its exact computation using depth and network size linear in $d$, it does provide an alternative construction which allows the approximation of $\max$ using a network of size $d\log(\log(d))$, but with a much smaller depth of only $\log(\log(d))$.

    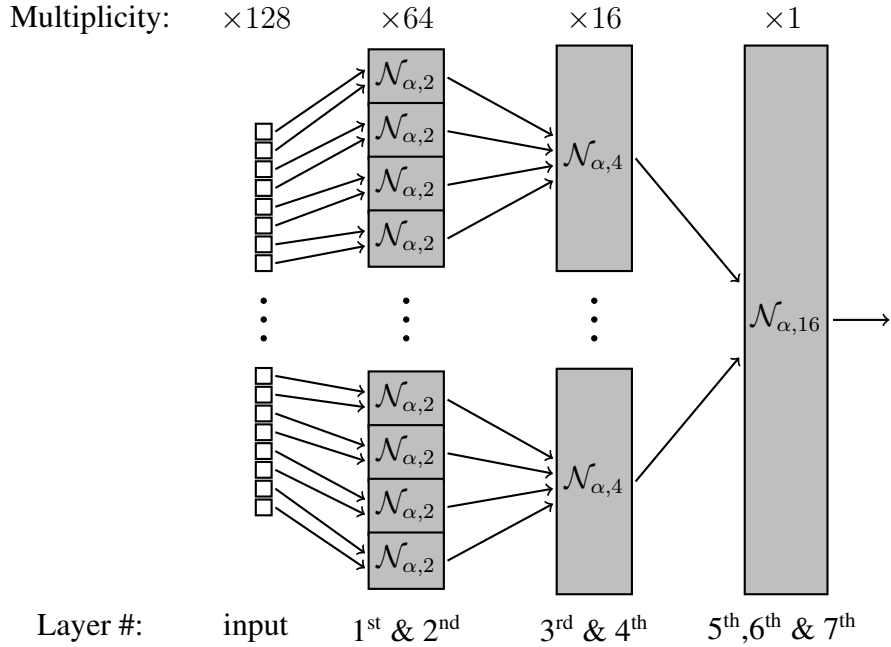
\begin{figure}[]
        \centering
        \scalebox{1}{
            \begin{tikzpicture}
                \foreach \x in {0,...,7,13,14,15,16,17,18,19,20}
                    \draw [thick,draw=black] (0,-1 * \x / 4) rectangle ++(0.21,0.21) node[midway] (i\x){};
                \foreach \y in {9,10,11}
                    \foreach \x in {0.105,2,4.5}
                        \filldraw[black] (\x,-1 * \y / 4 + 0.105) circle (1pt);
    
                \foreach \x in {0,...,3,6,7,8,9}
                    \draw [thick,draw=black,fill=lightgray] (1.5,-1 * \x / 1.4 + 0.45) rectangle ++(1,0.75)
                    node[midway] (\x1){$\Ncal_{\alpha,2}$};
    
                \draw [thick,draw=black,fill=lightgray] (4,-1.75) rectangle ++(1,3)
                    node[midway] (n21){$\Ncal_{\alpha,4}$};
    
                \draw [thick,draw=black,fill=lightgray] (4,-6.05) rectangle ++(1,3)
                    node[midway] (n22){$\Ncal_{\alpha,4}$};
    
                \draw [thick,draw=black,fill=lightgray] (6.5,-6.05) rectangle ++(1.1,7.3)
                    node[midway] (n3){$\Ncal_{\alpha,16}$};
                    
                \draw[thick,->] (i0.east) -- ([yshift=0.1cm]01.west);
                \draw[thick,->] (i1.east) -- ([yshift=-0.1cm]01.west);
    
                \draw[thick,->] (i2.east) -- ([yshift=0.1cm]11.west);
                \draw[thick,->] (i3.east) -- ([yshift=-0.1cm]11.west);
    
                \draw[thick,->] (i4.east) -- ([yshift=0.1cm]21.west);
                \draw[thick,->] (i5.east) -- ([yshift=-0.1cm]21.west);
    
                \draw[thick,->] (i6.east) -- ([yshift=0.1cm]31.west);
                \draw[thick,->] (i7.east) -- ([yshift=-0.1cm]31.west);
    
                \draw[thick,->] (i13.east) -- ([yshift=0.1cm]61.west);
                \draw[thick,->] (i14.east) -- ([yshift=-0.1cm]61.west);
    
                \draw[thick,->] (i15.east) -- ([yshift=0.1cm]71.west);
                \draw[thick,->] (i16.east) -- ([yshift=-0.1cm]71.west);
    
                \draw[thick,->] (i17.east) -- ([yshift=0.1cm]81.west);
                \draw[thick,->] (i18.east) -- ([yshift=-0.1cm]81.west);
    
                \draw[thick,->] (i19.east) -- ([yshift=0.1cm]91.west);
                \draw[thick,->] (i20.east) -- ([yshift=-0.1cm]91.west);
        
                \draw[thick,->] (01.east) -- (n21);
                \draw[thick,->] (11.east) -- (n21);
                \draw[thick,->] (21.east) -- (n21);
                \draw[thick,->] (31.east) -- (n21);
    
                \draw[thick,->] (61.east) -- (n22);
                \draw[thick,->] (71.east) -- (n22);
                \draw[thick,->] (81.east) -- (n22);
                \draw[thick,->] (91.east) -- (n22);
    
                \draw[thick,->] (n21.east) -- ([yshift=0.5cm]n3.west);
                \draw[thick,->] (n22.east) -- ([yshift=-0.5cm]n3.west);
    
                \draw[thick,->] (n3.east) -- ([xshift=0.75cm]n3.east);
    
                \node[draw=none] at (-2.2,1.6){Multiplicity:};
                \node[draw=none] at (0,1.6){$\times128$};
                \node[draw=none] at (2,1.6){$\times64$};
                \node[draw=none] at (4.5,1.6){$\times16$};
                \node[draw=none] at (7,1.6){$\times1$};
    
                \node[draw=none] at (-2.2,-6.5){Layer \#:};
                \node[draw=none] at (0,-6.5){input};
                \node[draw=none] at (2,-6.5){1$^\text{st}$~\&~2$^\text{nd}$};
                \node[draw=none] at (4.5,-6.5){3$^\text{rd}$~\&~4$^\text{th}$};
                \node[draw=none] at (7,-6.5){5$^\text{th}$,6$^\text{th}$~\&~7$^\text{th}$};
            \end{tikzpicture}
        }
        \caption{Sketch of the architecture $\Ncal_{\alpha,128}^3$ which approximates $f_{128}(\cdot)$ using depth $7$ and width $\approx128^{8/7}=256$. The multiplicity row at the top counts the number of components in each layer, and the layer \# row at the bottom indicates which layers participate in the computation of each component. The height of each $\Ncal_{\alpha,\cdot}$ component is roughly proportional to its width. Each pair of hidden layers increases the batch size on which the maxima are computed quadratically, while maintaining the width of the network roughly the same across all hidden layers. This results in a double exponential decay of the batch size, allowing an approximation of the maximum with depth $\Ocal(\log(\log(d)))$.}
        \label{fig:deep}
    \end{figure}

    Of particular interest is the following corollary, which provides an approximation guarantee in the case where the data are sampled uniformly from a hypercube. Most importantly, in such a case we can guarantee an approximation to accuracy $\varepsilon>0$ using a network with weights that scale polynomially with $d$ and linearly with $1/\varepsilon$. This property will turn out to be useful in the next section where we will show lower bounds for approximating the maximum function with respect to the uniform distribution.

    \begin{corollary}\label{cor:uniform_approximation}
        For any $\varepsilon>0$ and naturals $d\ge58$ and $1\le k\le\lceil\log(\log(d)+1)\rceil$, there exists a ReLU network $\Ncal$ of depth $2k+1$, width at most $20d^{1+\beta(k)}$ and weights of magnitude $\Ocal\p{\frac{d^4R^2}{\varepsilon}}$ such that
        \[
            \E_{\bx\sim\Ucal\p{[0,R]^d}}\pcc{\p{\Ncal(\bx) - f_d(\bx)}^2} \le \varepsilon.
        \]
    \end{corollary}


    Having stated our main positive approximation results, we now turn to specify the constructions used to achieve them. Beginning with defining the depth 3 network used in \thmref{thm:depth3_approx}, we have the following architecture.

    \begin{definition}\label{def:depth3}[Depth 3 approximation]
        Given a weight upper bound $\alpha>0$ and input dimension $d$, we define the following depth 3 width $d(d+1)$ network which approximates $f_d(\bx)$:
        \[
            \Ncal_{\alpha,d}(\bx)\coloneqq \sum_{i=1}^d\p{
            \relu{\relu{x_{i}}-\sum_{\substack{j=1\\j\neq i}}^{d}\relu{\alpha x_{j}-\alpha x_{i}}}
            -\relu{\relu{-x_{i}}-\sum_{\substack{j=1\\j\neq i}}^{d}\relu{\alpha x_{j}-\alpha x_{i}}}
            }.
        \]
    \end{definition}

    We remark that we occasionally omit the dimension subscript whenever clear from context, and we note that the above architecture can be realized using a width $d(d+1)$ ReLU network since the inner sum terms are identical, and thus computing both requires only $d-1$ neurons. 

    Next, we define the architecture which approximates the function $f_d(\bx)$ using depth $2k+1$ and width at most $20d^{1+\beta(k)}$, achieving the approximating result stated in \thmref{thm:deep_approx}.

    \begin{definition}\label{def:deep}[Depth $2k+1$ approximation]
        Given a weight upper bound $\alpha>0$ and input dimension $d$, we define the following depth $2k+1$ width at most $20d^{1+\beta(k)}$ network $\Ncal_{\alpha,d}^k$ which approximates $f_d(\bx)$ in a recursive manner:

        \begin{itemize}
            \item 
            For $k=1$, we have $\Ncal_{\alpha,d}^1\equiv\Ncal_{\alpha,d}$.
            
            \item
            For integer $k>1$, we partition the input into $\ceil{d^{1-\beta(k)}}$ batches, each of size at most $\ceil{d^{\beta(k)}}$. For each batch, the first two hidden layers compute the function $\Ncal_{\alpha,\ceil{d^{\beta(k)}}}$. The output over all the batches is then fed into the sub-network $\Ncal_{\alpha,\ceil{d^{1-\beta(k)}}}^{k-1}$ which consists the remaining layers of the network $\Ncal_{\alpha,d}^k$.

        \end{itemize}

    \end{definition}

    \section{Approximation lower bounds}\label{sec:negative}

    Having shown a positive approximation result for the maximum function in the previous section, we now turn to complement our approximation upper bounds with lower bounds in this section.

    \subsection{Improving accuracy requires increasing width for depth 2 networks}



    Before presenting our main theorem for this subsection, we first state the following very mild assumption that we use which is adopted from \citet{eldan2016power}:

    \begin{assumption}[Polynomially-bounded activation]\label{asm:poly_bounded}
        The activation function $\sigma$ is Lebesgue measurable and satisfies
        \[
            \abs{\sigma(x)}\le C_{\sigma}\p{1+\abs{x}^{\alpha_{\sigma}}}
        \]
        for all $x\in\reals$ and for some constants $C_{\sigma},\alpha_{\sigma}>0$.
    \end{assumption}

    Our depth 2 lower bound is the following:
    \begin{theorem}\label{thm:depth2}
        Let $\ell\ge1$ be arbitrary and suppose that $\sigma$ satisfies \asmref{asm:poly_bounded}. Then there exist constants $c_1,c_2>0$ which depend solely on $\sigma$ such that for all dimensions $d\ge c_1$, a $\sigma$ depth 2 neural network $\Ncal$ of width at most $d^{\ell}$ and with weights bounded by $\Ocal(\exp(\Ocal(d)))$ must satisfy
        \[
            \E_{\bx\sim\Ucal\p{\pcc{0,1}^d}}\pcc{\p{\Ncal(\bx) - f_d(\bx)}^2} > \Omega\p{d^{-c_2\cdot\ell}}.
        \]
    \end{theorem}

    The proof of the above theorem, which appears in \appref{app:depth2_proof}, builds on the proof technique introduced in \citet{eldan2016power}. Roughly speaking, they build on the important observation that a neural network $\Ncal$ approximates a function $f$ if and only if the Fourier transform of $\Ncal$ approximates the Fourier transform of $f$. Our main technical contribution here is the computation of the Fourier transform of the maximum function and showing that it has sufficient $L_2$ mass far away from the origin which is sufficiently spread. This, in turn, shows that the support of the Fourier transform of a neural network (which is a linear combination of the Fourier transform of its activation function) must be contained inside a $d$-dimensional union of `Gaussian tubes'. Namely, under the assumption of exponentially bounded weights, the approximation contribution of each neuron is negligible outside a union of tubes with bounded radius. This entails that to approximate the Fourier transform of the maximum function, one must use sufficiently many neurons in order to be able to capture its non-trivial structure which is sufficiently spread across the domain of approximation.

    While the above theorem establishes a polynomial rather than exponential separation between depths 2 and 3, such a gap is nevertheless significant since it provides a compelling practical example where depth is more beneficial than width: Modern machine learning problems are often high-dimensional, hence even such polynomial gaps quickly translate into a significant practical advantage in the size of the required network. We further remark that our assumption that the approximating network has exponentially bounded weights is mild and very reasonable. This follows from the fact that approximations with weights that have exponential magnitude are known to be difficult to learn using stable gradient descent \citep{safran2022optimization}, so from a more practical perspective having exponentially bounded weights and having unbounded weights is equivalent. Lastly, as we also pointed out earlier, our result also implies a lower bound with constant accuracy by rescaling our domain of approximation polynomially with $d$. Performing such a manipulation is justified since our upper bounds from the previous section are not sensitive to scaling of the domain and merely require that the weights of the approximating network would also scale appropriately (\corollaryref{cor:uniform_approximation}). This enables us to show a separation in which a depth 2 network cannot approximate the maximum to better than constant accuracy, whereas a depth 3 network with a fixed width of $d(d+1)$ that does not depend on the desired accuracy can achieve arbitrarily good accuracy. In contrast, known results in the literature do require that the width of the depth 3 network would scale with the accuracy parameter (e.g.\ \citet{eldan2016power,daniely2017depth} and results that build on their technique -- see related work subsection).
    
    \subsection{Depth 3 requires $\Omega\p{d^2}$ width}

    In this subsection, we show that approximating the maximum function using depth 3 ReLU networks with weights bounded by $\Ocal(\exp(\Ocal(d)))$ requires width at least $\Omega(d^2)$. Our main result in this subsection is the following.

    \begin{theorem}\label{thm:depth3_lb}
        Suppose that $\Ncal$ is a depth 3 ReLU network of width at most $\frac{d^2}{5}$ and with weights bounded by $\Ocal(\exp(\Ocal(d)))$. Then there exist absolute constants $c_1,c_2>0$ such that for all $d\ge c_1$
        \[
            \E_{\bx\sim\Ucal\p{\pcc{0,1}^d}}\pcc{\p{\Ncal(\bx) - f_d(\bx)}^2} > \Omega\p{d^{-c_{2}}}.
        \]
    \end{theorem}

    The proof of the above theorem, which appears in \appref{app:depth3_lb}, exploits the structure of the maximum function which computes a lower dimensional version of itself on every subset of its inputs. Using a combinatorial argument, we show that with fewer than $d^2/5$ neurons in the first hidden layer, our approximating network must be able to approximate the maximum over three inputs well on a non-negligible subset of its domain with its remaining layers (see \figref{fig:triangle_arg} for a more detailed explanation of the proof technique). Since in the previous subsection we have shown a lower bound on the approximation capabilities of depth 2 ReLU networks for the maximum function, this implies that if the second hidden layer is also at most $d^2/5$, then we cannot obtain a good approximation.

    \begin{figure}
        \centering
        \begin{tikzpicture}
            \filldraw (0,2) circle (3pt) node[above left] {$1$};
            \filldraw (0,0) circle (3pt) node[below left] {$3$};
            \filldraw (2,2) circle (3pt) node[above right] {$2$};
            \filldraw (2,0) circle (3pt) node[below right] {$4$};
            
            \draw[very thick] (0,0)--(0,2);
            \draw[very thick] (2,0)--(0,2);
            \draw[very thick] (2,2)--(0,2);
            \draw[very thick] (0,0)--(2,2);
            
            \node[] at (5.5,1.6){$\bw_1\coloneqq \p{0,0,1,-1}$};
            \node[] at (5.5,1){$\bw_2\coloneqq \p{0,-1,0,1}$};
            \node[] at (5.5,0.4){$\bw_3\coloneqq \p{0,1,0,-1}$};
        \end{tikzpicture}
        \caption{An informal simplication of the main proof idea behind \thmref{thm:depth3_lb}. To identify a region in the domain of approximation where $f_d(\cdot)$ is poorly approximated, we construct a graph $G$ as follows: We begin with the complete graph on $d$ vertices (a vertex for each input), and we remove an edge $(i,j)$ if there exists a neuron in the first hidden layer whose weight vector is all zeros except for the coordinates $i,j$. In the example portrayed in the figure, for weight vectors $\bw_1,\bw_2,\bw_3$ we remove the edge $(3,4)$ due to $\bw_2$, and the edge $(2,4)$ due to either of $\bw_2,\bw_3$. Loosely speaking, whenever the first hidden layer contains a neuron with an all-zero weight vector except for two coordinates that equal $-1$ and $1$, then this neuron is able to capture the non-linearity of $f_d(\cdot)$ associated with the two non-zero coordinates when one overtakes the other in value. Since the width of the network is at most $d^2/5$, this guarantees that $G$ contains at least $\binom{d}{2}-d^2/5>d^2/4$ edges. By Mantel's theorem (see \thmref{thm:mantel}), $G$ must contain a triangle, which implies the existence of a $3$-dimensional sub-cube where the non-linearities in the first hidden layer are redundant. This effectively reduces the approximation of $f_d(\cdot)$ using a depth 3 network to the approximation of $f_3(\cdot)$ using a depth 2 network.}
        \label{fig:triangle_arg}
    \end{figure}
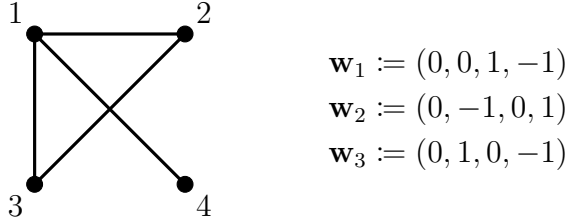

    We remark that together with \thmref{thm:depth3_approx}, we establish tight bounds on the capability of depth 3 ReLU networks to approximate the maximum function (up to constant factors). It is also interesting to note that other existing lower bound techniques in the literature such as region counting arguments (c.f.\ \citet{telgarsky2016benefits}) when applied to the maximum function yield a far weaker lower bound for depth 3 networks, since the maximum function consists of $d$ different linear regions, a number which is attainable by merely using $\log(d)$ neurons. In contrast, our lower bound of $\Omega(d^2)$ highlights an inherent limitation of depth 3 ReLU networks for capturing the particular structure of the maximum function. When combined with \thmref{thm:deep_approx} for $k=2$, our theorem also implies a polynomial depth separation between depths 3 and 5 where the former requires width $\Omega(d^2)$, yet the latter can approximate with width $\Ocal(d^{4/3})$.

    \subsection{Width of at least $d$ is necessary}

    Having shown lower bounds for depth 2 and 3 ReLU networks, we now turn to show a general width $d$ lower bound requirement for approximating the maximum function.
    \begin{theorem}\label{thm:size_d_lb}
        Let $\Ncal$ be neural network employing any activation function and having first hidden layer width of at most $d-1$. Then
        \[
            \E_{\bx\sim\Ucal\p{\pcc{0,1}^d}}\pcc{\p{\Ncal(\bx)-f_d(\bx)}^2} \ge \Omega\p{d^{-4.5}}.
        \]
    \end{theorem}
    The proof of the above theorem, which appears in \appref{app:size_d_lb_proof}, relies on the observation that having fewer than $d$ neurons in the first hidden layer implies that the linear transformation defined by them has a non-trivial kernel, and therefore establishes the existence of some direction in the domain of approximation where the function computed by the network is constant. Since the maximum function typically does not remain constant in most directions, this results in a non-trivial approximation error.

    We remark that the exponent of $-4.5$ in the target accuracy can possibly be improved somewhat, but in any case it must be strictly positive, since a network with a single neuron which computes the constant value function $1-1/d$ will achieve better than constant accuracy over the domain $[0,1]^d$. Moreover, we note that this result also immediately implies a size $d$ lower bound. Together with \thmref{thm:deep_approx}, this establishes size bounds for approximating the maximum function using ReLU networks that are tight up to a factor of $\Ocal(\log(\log(d)))$.

    \section{Summary}\label{sec:summary}

    We have shown that the maximum function can be gradually approximated more efficiently by increasing the depth of the approximating ReLU network. This holds under an appropriate (but mostly mild) assumption on the distribution of the data (\asmref{asm:dist}). Interestingly, the width in our positive approximation results does not scale with the desired target accuracy, but rather by increasing the magnitude of the weights of the approximating network we can obtain an arbitrarily good approximation. Assuming exponentially bounded weights, we show a polynomial lower bound on the required width when approximating the maximum function using depth 2, and a quadratic lower bound on the width required for approximating using depth 3. Additionally, we also provide a general width $d$ lower bound for approximating the maximum function using neural networks of any depth or with any activation function. Our results establish a partial depth hierarchy for approximating a simple target function and with respect to the uniform distribution on a hypercube, which provides a more grounded example for the benefits of depth compared to previous results which make more stylized assumptions on the problem.

    Our analysis leaves several important open questions. First, our lower bound for depth 2 is (inverse) polynomial in the desired accuracy, which becomes polynomial in the input dimension $d$ if we scale the domain with $d$. However, it is not clear what is the optimal rate at which a depth 2 network can approximate the maximum function, and whether this quantity is polynomial or rather exponential in the input dimension. Second, despite our tight $\Theta(d^2)$ bound on the width for approximating the maximum using depth 3, our lower bound for deeper architectures is only linear, which leaves open the question of showing superlinear width lower bounds for depths larger than $3$. Moreover, our upper bounds essentially suggest that depth $2k+1$ and depth $2k+2$ ReLU network approximations require the same width, up to constant factors. It is therefore natural to ask whether one can improve our depth $2k+1$ upper bounds to apply to depth $k+1$ instead. Finally, our analysis opens an avenue for novel optimization-based separations for the maximum function. Proving that indeed deep architectures are capable of learning the representations constructed by our upper bounds (efficiently, from finite data) using standard techniques such as gradient descent is a tantalizing future research direction. Such a result holds the potential to establish an optimization-based depth hierarchy for learning the maximum function, exemplifying the superiority of depth in a simple and natural problem setting.

    \subsection*{Acknowledgements}

    Part of this work was done while the second author was visiting the Simons Institute for the Theory of Computing. Their hospitality is greatly acknowledged.
    

    
    \bibliographystyle{abbrvnat}
    \bibliography{citations}

    \appendix
    
    




    \section{Distributions that satisfy \asmref{asm:dist}}\label{app:delta_sep_dists}

    In this appendix, we exemplify two instances of distribution that satisfy \asmref{asm:dist}. While the examples presented here are quite broad, they are in no way exhaustive, and a far richer family of distributions can be shown to satisfy the assumption.

    \begin{theorem}\label{thm:case1}
        Suppose that $X$ is absolutely continuous with bounded density and has a finite second moment. Then $\bX\coloneqq(X_1,\ldots,X_d)$ satisfies \asmref{asm:dist}, where each $X_i$ is an i.i.d $X$ random variable.
    \end{theorem}

    \begin{proof}
        Beginning with \itemref{item:dist1}, we have
        \[
            \E_{\bX}\pcc{\norm{\bX}_{\infty}^2} \le \E_{\bX}\pcc{\norm{\bX}_{1}^2} \le d\sum_{i=1}^d\E_{\bX}\pcc{X_i^2} = d^2\E_{X}\pcc{X^2} <\infty.
        \]
        In the above, the second inequality is an application of Cauchy-Schwartz to the vectors $(1,\ldots,1)$ and $(X_1,\ldots,X_d)$ and due to the linearity of expectation, and the last inequality is due to our assumption that $X$ has a finite second moment.

        Moving on to \itemref{item:dist2}, assuming $X$ has density $f$ satisfying $\sup_{x\in\reals}f(x)\le C$ for some $C>0$, we have that the density of the ratio distribution between two different coordinates of $\bX$ satisfies
        \[
            f_R(r)=\int_{\reals}\abs{x}f(r\cdot x)f(x)dx \le C\cdot\int_{\reals}\abs{x}f(x)dx = C\E_{X}\pcc{\abs{X}} \le C\sqrt{\E_{X}\pcc{X^2}} < \infty,
        \]
        where the penultimate inequality follows from Jensen's inequality applied to the function $x\mapsto x^2$. The above implies that for any two coordinates $X_i,X_j$, we have
        \[
            \lim_{\delta\to0}\frac{X_i}{X_j}\in[1-\delta,1+\delta] = 0,
        \]
        therefore by taking a union bound over all $\le d^2$ pairs of coordinates we have that
        \[
            \lim_{\delta\to0}\pr_{\bX\sim\Dcal}\pcc{\bX\notin S_{\delta}}=0.
        \]
    \end{proof}

    \begin{theorem}
        Suppose that $\bX\coloneqq(X_1,\ldots,X_d)$ satisfies \itemref{item:dist1} in \asmref{asm:dist}. Then the vector $\bX+\bY$ satisfies \asmref{asm:dist}, where  $\bY\coloneqq(Y_1,\ldots,Y_d)$ is an i.i.d noise vector such that $Y_i$ is absolutely continuous, with bounded density, and $0<\E\pcc{Y_i^2}<\infty$.
    \end{theorem}

    \begin{proof}
       To show that \itemref{item:dist1} holds, we have by \thmref{thm:case1} that $\bY$ satisfies \itemref{item:dist1}. This entails that
        \[
            \norm{\bX+\bY}_{\infty} \le \norm{\bX}_{\infty} + \norm{\bY}_{\infty} < \infty.
        \]
        To show that \itemref{item:dist2} is satisfied, consider any coordinates $i\neq j$. Let $z_i\coloneqq x_i+y_i$ denote the realizations of $Z_i\coloneqq X_i+Y_i$. Then we have that for any realization $x_j$ of $X_j$, it must hold that $y_j$ falls within an interval of length at most $2\delta$ to have that
        \[
            z_j\in[z_i-\delta,z_i+\delta] \iff \abs{z_j-z_i}\le\delta.
        \]
        Since
        \[
            \sup_{a\in\reals}\pr\pcc{Y_j\in[a,a+2\delta]} = \sup_{a\in\reals}\int_a^{a+2\delta}f(x)dx\le 2\delta C,
        \]
        where $f$ is the density of $Y$ which satisfies $\sup_xf(x)\le C$ by assumption, we obtain
        \[
            \pr\pcc{\abs{z_j-z_i}\le\delta} \le 2\delta C.
        \]
        Since $Z_i$ has bounded first and second moments for all $i$ by assumption, we have from Chebyshev's inequality that there exists some $M_{\delta}>0$ such that $\pr\pcc{\norm{\bX}_{\infty}\le M_{\delta}}\ge1-\delta$.
        Since \lemref{lem:add_mul_separated} guarantees that if $\abs{z_j-z_i}>\delta$ then $\bz\in S_{\delta/M_{\delta}}$, we have from a union bound taken over all $\le d^2$ pairs of coordinates and the event where $\norm{\bX}_{\infty}$ is bounded that
        \[
            \lim_{\delta\to0}\pr_{\bX\sim\Dcal}\pcc{\bX\notin S_{\delta}}=0.
        \]
    \end{proof}
    
    \section{Proofs from \secref{sec:positive}}\label{app:positive_proofs}

    To prove \thmref{thm:depth3_approx} and \thmref{thm:deep_approx}, we would first need the following lemmas and proposition:
    
    \begin{lemma}\label{lem:polytope_approx}
        Given a real $\alpha>0$ and integer $i\ge1$, let
        \[
            n_{\alpha,i}^+(\bx)\coloneqq \relu{\relu{x_{i}}-\sum_{\substack{j=1\\j\neq i}}^{d}\relu{\alpha x_{j}-\alpha x_{i}}},\hskip 0.3cm
            n_{\alpha,i}^-(\bx)\coloneqq -\relu{\relu{-x_{i}}-\sum_{\substack{j=1\\j\neq i}}^{d}\relu{\alpha x_{j}-\alpha x_{i}}}.
        \]
        Then
        \[
            n_{\alpha,i}^+(\bx)
            =\begin{cases}
                x_{i}, & \text{If } f_d(\bx)=x_{i}\text{ and } x_{i}\ge 0,\\
                0, & \text{If } f_d(\bx)>x_{i}\text{ and } \bx\in S_{1/\alpha},
            \end{cases}
        \]
        and
        \[
            n_{\alpha,i}^-(\bx)
            =\begin{cases}
                x_{i}, & \text{If } f_d(\bx)=x_{i}\text{ and } x_{i}\le 0,\\
                0, & \text{If } f_d(\bx)>x_{i}\text{ and } \bx\in S_{1/\alpha}.
            \end{cases}
        \]
    \end{lemma}

    \begin{proof}
        We will only focus on the proof for $n_{\alpha,i}^+$ since the analysis is symmetric for $n_{\alpha,i}^-$. 

        Suppose that $f_d(\bx)=x_{i}$. Then $x_{j}<x_{i}$ for all $j\neq i$, implying that $\alpha x_{j}-\alpha x_{i}<0$ and thus
        \[
            n_{\alpha,i}^+(\bx) = \relu{\relu{x_{i}}-\sum_{\substack{j=1\\j\neq i}}^{d}\relu{\alpha x_{j}-\alpha x_{i}}} = \relu{x_{i}}.
        \]
        Suppose that $f_d(\bx)>x_{i}$. Then if $x_{i}\le0$ we immediately have that $n_{\alpha,i}^+(\bx)=0$. Assuming $x_{i}>0$ and letting $j\coloneqq\argmax{i\in[d]}~x_{i}$, we have that $x_{j}>x_{i}>0$. Next, from the assumption $\bx\in S_{1/\alpha}$ we obtain
        \[
            \frac{x_{j}}{x_{i}}>1+\frac{1}{\alpha},
        \]
        which entails
        \[
            \alpha x_{j}-\alpha x_{i} > x_{i},
        \]
        implying that
        \[
            \relu{x_{i}}-\sum_{\substack{i=1\\i\neq i}}^{d}\relu{\alpha x_{i}-\alpha x_{i}} \le \relu{x_{i}}-\relu{\alpha x_{j}-\alpha x_{i}} <  x_{i} - x_{i} = 0,
        \]
        and thus
        \[
            n_{\alpha,i}^+(\bx)=0.
        \]
    \end{proof}

    \begin{lemma}\label{lem:depth3_properties}
         Given a real $\alpha>0$, we have
         \[
            \Ncal_{\alpha}(\bx)
            = f_d(\bx),\hskip 0.3cm \forall \hskip 0.1cm \bx\in S_{1/\alpha},
         \]
         and
         \[
            \abs{\Ncal_{\alpha}(\bx)}
            \le \norm{\bx}_1,\hskip 0.3cm \forall \hskip 0.1cm \bx\in\reals^d.
         \]
    \end{lemma}

    \begin{proof}
        By the definition of $\Ncal_{\alpha}$ and \lemref{lem:polytope_approx}, we have for all $\bx\in S_{1/\alpha}$ that
        \[
            \Ncal_{\alpha}(\bx) = \sum_{i=1}^d \p{n_{\alpha,i}^+(\bx) + n_{\alpha,i}^-(\bx)} = \sum_{i=1}^d
            x_{i}\cdot\one{f_d(\bx)=x_{i}}
            = f_d(\bx).
        \]
        Assuming any arbitrary $\bx\in\reals^d$, we have by the definitions of $n_{\alpha,i}^{\pm}$ that
        \[
            \abs{n_{\alpha,i}^{\pm}(\bx)} = \abs{\relu{\pm x_{i}}-\sum_{\substack{j=1\\i\neq i}}^{d}\relu{\alpha x_{i}-\alpha x_{i}}} \le \relu{\pm x_{i}} \le \abs{x_{i}},
        \]
        therefore by the definition of $\Ncal_{\alpha}$ and the fact that at most one of $n_{\alpha,i}^{\pm}$ is non-zero, we have
        \[
            \abs{\Ncal_{\alpha}(\bx)} \le \sum_{i=1}^d \abs{n_{\alpha,i}^+(\bx) + n_{\alpha,i}^-(\bx)}\le \norm{\bx}_1.
        \]
    \end{proof}

    \begin{proposition}\label{prop:deep_approx}
        Given any real $\alpha>0$, we have
         \[
            \Ncal_{\alpha}^k(\bx)
            = f_d(\bx),\hskip 0.3cm \forall \hskip 0.1cm \bx\in S_{1/\alpha},
         \]
         and
         \[
            \abs{\Ncal_{\alpha}^k(\bx)}
            \le d\cdot f_d(\bx),\hskip 0.3cm \forall \hskip 0.1cm \bx\in\reals^d.
         \]
    \end{proposition}

    \begin{proof}
        The proof follows by induction on $k$. The base case $k=1$ follows immediately from \lemref{lem:depth3_properties}. In what follows, given a natural $k$, recall that we use the shorthand $\beta(k)\coloneqq\frac{1}{2^k-1}$. 

        For the inductive step, assume the induction hypothesis for $k$, and consider the network $\Ncal_{\alpha}^{k+1}$. Since any sub-vector of a $\delta$-separated vector is also $\delta$-separated for all $\delta>0$, we have from \lemref{lem:depth3_properties} that the output of the second hidden layer of $\Ncal_{\alpha}^{k+1}$ is the maximum over each of the $\ceil{d^{1-\beta(k)}}$ batches. Applying the inductive hypothesis on the sub-network consisting of layers $3$ to $2k+3$, the network outputs $f_d(\bx)$.
        
        For the second part of the proposition, partition the input $\bx$ into $\ceil{d^{1-\beta(k)}}$ batches such that the vector of inputs in each batch has dimension at most $\ceil{d^{\beta(k)}}$. Then by \lemref{lem:depth3_properties}, each coordinate in the output of the second hidden layer of $\Ncal_{\alpha}^{k+1}$ is upper bounded by
        \[
            \p{\norm{\bx_1}_1,\ldots,\norm{\bx_{\ceil{d^{1-\beta(k)}}}}_1}.
        \]
        Applying the induction hypothesis on the sub-network consisting of layers $3$ to $2k+3$, we obtain
        \[
            \abs{\Ncal_{\alpha}^k(\bx)} \le \sum_{i=1}^{\ceil{d^{1-\beta(k)}}}\norm{\bx_i}_1
            = \norm{\bx}_1 \le d\cdot f_d(\bx).
        \]
    \end{proof}

    With the above lemmas and proposition, we are now ready to prove the theorems. Since the proof of \thmref{thm:deep_approx} follows mainly by induction and since \thmref{thm:depth3_approx} consists the base case for the induction, it would be convenient to prove both theorems simultaneously.
    
    \begin{proof}[Proofs of \thmref{thm:depth3_approx} and \thmref{thm:deep_approx}]
        Recall we use the shorthand $\beta(k)\coloneqq\frac{1}{2^k-1}$ for any natural $k\ge1$. We begin with asserting the size of the approximating network. We have that $\Ncal_{\alpha}^k$ has depth $2k+1$ and weights of magnitude at most $\alpha$ by definition (note that the output neuron of $\Ncal_{\alpha}$ has weights of magnitude exactly $1$, and therefore composing its weights with the subsequent layer's weights does not increase the magnitude). 
        
        Next, we bound the width of $\Ncal_{\alpha}^k$ using induction. To this end, we will show that for all natural $k\ge1$ we have an upper bound on the width of
        \[
            \prod_{i=1}^k\p{1+\frac{2}{i^3}}^2d^{1+\beta(k)}.
        \]
        By \lemref{lem:prod}, we have that $\prod_{i=1}^{\infty}\p{1+\frac{2}{i^3}}^2\le 20$, thus the above implies the desired upper bound on the width. 
        
        The base case is immediate since $\Ncal_{\alpha}^1\equiv\Ncal_{\alpha}$ which has width exactly $d(d+1)\le2d^2$. Assume the inductive hypothesis for $k$, and consider the network $\Ncal_{\alpha}^{k+1}$. Its first two hidden layers consist of $\ceil{d^{1-\beta(k+1)}} \le d^{1-\beta(k+1)}+1$ batches of $\Ncal_{\alpha,\ceil{d^{\beta(k+1)}}}$, each of which having width at most
        \[
            \ceil{d^{\beta(k+1)}}^2+\ceil{d^{\beta(k+1)}} \le \p{d^{\beta(k+1)}+1}^2+d^{\beta(k+1)}+1 = d^{2\beta(k+1)} + 3d^{\beta(k+1)} + 2,
        \]
        for a total width upper bound of
        \[
            \p{d^{1-\beta(k+1)}+1}\p{d^{2\beta(k+1)} + 3d^{\beta(k+1)} + 2} \le 12d^{1+\beta(k+1)},
        \]
        thus implying an upper bound on the width of
        \[
            12d^{1+\beta(k+1)} \le \prod_{i=1}^{k+1}\p{1+\frac{2}{i^3}}^2d^{1+\beta(k+1)},
        \]
        since $k\ge1$ and the product over the first two elements is at least $14$. Moving on to bound the width of the remaining layers, we have by definition that layers $3$ to $2k+3$ is the network $\Ncal_{\alpha,\ceil{d^{1-\beta(k+1)}}}^{k}$, which by the induction hypothesis has width at most
        \begin{align*}
            \prod_{i=1}^{k}\p{1+\frac{2}{i^3}}^2 \ceil{d^{1-\beta(k+1)}}^{1+\beta(k)} &\le \prod_{i=1}^{k}\p{1+\frac{2}{i^3}}^2 \p{d^{1-\beta(k+1)}+1}^{1+\beta(k)}\\
            &\le\prod_{i=1}^{k}\p{1+\frac{2}{i^3}}^2 \p{\p{1+\frac{2}{(k+1)^3}}d^{1-\beta(k+1)}}^{1+\beta(k)}\\
            &\le\prod_{i=1}^{k}\p{1+\frac{2}{i^3}}^2 \p{1+\frac{2}{(k+1)^3}}^2\p{d^{1-\beta(k+1)}}^{1+\beta(k)}\\
            &=\prod_{i=1}^{k+1}\p{1+\frac{2}{i^3}}^2d^{1+\beta(k+1)}.
        \end{align*}
        In the above, the second inequality follows from \lemref{lem:technical_inequality} by our assumption that $d\ge58$ and $1\le k\le\lceil\log(\log(d)+1)\rceil$, and the third inequality follows from the fact that $\beta(k)\le1$ for all natural $k\ge1$. We thus conclude that $\Ncal_{\alpha}^{k}$ has width at most $20d^{1+\beta(k)}$.

        Turning to bound the approximation error of $\Ncal_{\alpha}^k$, we first have from \asmref{asm:dist} that there exists some $\delta_0>0$ such that
        \begin{equation}\label{eq:delta0}
            \pr_{\bx\sim\Dcal}\pcc{\bx\notin S_{\delta_0}} \le \frac{\varepsilon}{(d+1)^2\E_{\bx\sim\Dcal}\pcc{\norm{\bx}_{\infty}^2}}.
        \end{equation}
        We now compute using the law of total expectation
        \begin{align*}
            \E_{\bx\sim\Dcal}\pcc{\p{\Ncal_{1/\delta_0}^k(\bx)-f_d(\bx)}^2} &= \E_{\bx\sim\Dcal}\pcc{\p{\Ncal_{1/\delta_0}^k(\bx)-f_d(\bx)}^2|\bx\in S_{\delta_0}}\cdot\pr_{\bx\sim\Dcal}\pcc{\bx\in S_{\delta_0}}\\
            &\hskip 0.5cm + \E_{\bx\sim\Dcal}\pcc{\p{\Ncal_{1/\delta_0}^k(\bx)-f_d(\bx)}^2|\bx\notin S_{\delta_0}}\cdot\pr_{\bx\sim\Dcal}\pcc{\bx\notin S_{\delta_0}}\\
            &= \E_{\bx\sim\Dcal}\pcc{\p{\Ncal_{1/\delta_0}^k(\bx)-f_d(\bx)}^2|\bx\notin S_{\delta_0}}\cdot\pr_{\bx\sim\Dcal}\pcc{\bx\notin S_{\delta_0}}\\
            &\le \E_{\bx\sim\Dcal}\pcc{(d+1)^2\norm{\bx}_{\infty}^2} \cdot \frac{\varepsilon}{(d+1)^2\E_{\bx\sim\Dcal}\pcc{\norm{\bx}_{\infty}^2}} = \varepsilon,
        \end{align*}
        where the second equality follows from \propref{prop:deep_approx}, and the inequality also follows from \propref{prop:deep_approx} and \eqref{eq:delta0}.

        Lastly, we verify that $\lceil\log(\log(d)+1)\rceil = \Ocal(\log(\log(d)))$. We have
        \[
            20d^{1+\frac{1}{2^k-1}} \le 20d^{1+\frac{1}{\log(d)}} = 20d\cdot d^{\frac{1}{\log(d)}} = 20d\cdot 2^{\log(d)\frac{1}{\log(d)}} = 40d,
        \]
        thus for this choice of $k$ we have a network of depth $\Ocal(\log(\log(d)))$ and width $\Ocal(d)$ which approximates $f_d(\cdot)$, concluding the proof of the theorem.
    \end{proof}

    \begin{proof}[Proof of \corollaryref{cor:uniform_approximation}]
        To prove the corollary, we need only show that $\Dcal\sim\Ucal\p{[0,R]^d}$ satisfies \asmref{asm:dist} and compute the $\delta_0>0$ for which \eqref{eq:delta0} holds. Starting with \itemref{item:dist1}, it is trivial that
        \begin{equation}\label{eq:uniform_second_moment}
            \E_{\bx\sim\Ucal\p{[0,R]^d}}\pcc{\norm{\bx}_{\infty}^2} \le R^2.
        \end{equation}
        Moving on to \itemref{item:dist2}, let $\delta>0$ be some arbitrary real number. Drawing any $x_i\sim\Ucal([0,R])$ for some $i\in[d]$, we have with probability at most $\frac{2\delta}{R}$ that it is within distance of at most $\delta$ from any other $x_j$, $j<i$. By a union bound taken over the distances from all coordinates, any freshly sampled coordinate is within distance at least $\delta$ from all the other coordinates with probability at least $1-\frac{2d\delta}{R}$. Taking another union bound over drawing each coordinate sufficiently far and using \lemref{lem:add_mul_separated} with the fact that $\pr\pcc{\abs{x_i}\le R}=1$ for all $i$, we have that 
        \[
            \pr_{\bx\sim\Ucal\p{[0,R]^d}}\pcc{\bx\notin S_{\delta/R}}\le\frac{2d^2\delta}{R},
        \]
        which by a change of variables $\delta_0=\delta/R$ implies
        \[
            \pr_{\bx\sim\Ucal\p{[0,R]^d}}\pcc{\bx\notin S_{\delta_0}}\le2d^2\delta_0.
        \]
        It is only left to compute the $\delta_0$ for which \eqref{eq:delta0} holds. To this end, we wish to find $\delta>0$ such that
        \[
            2d^2\delta \le \frac{\varepsilon}{(d+1)^2R^2} \le \frac{\varepsilon}{(d+1)^2\E_{\bx\sim\Ucal\p{[0,R]^d}}\pcc{\norm{\bx}_{\infty}^2}},
        \]
        where the second inequality uses \eqref{eq:uniform_second_moment}. Solving the above for $\delta$, we have $\delta_0=\Omega\p{\frac{\varepsilon}{d^4R^2}}$, implying a weight upper bound of $\Ocal\p{\frac{d^4R^2}{\varepsilon}}$ and concluding the proof of the corollary.
    \end{proof}

    \section{Proofs from \secref{sec:negative}}

    \subsection{Proof of \thmref{thm:depth2}}\label{app:depth2_proof}

    The following proposition is key in the proof of the theorem.
    
    \begin{proposition}\label{prop:finite_d}
        For each fixed dimension $d\geq 2$, the squared $L_2$ error of approximating the function $\max\{0,x_1,x_2,\ldots,x_d\}$ on the unit Gaussian, using $n$ neurons in a depth 2 $\sigma$ network satisfying \asmref{asm:poly_bounded}, with coefficients of size at most $s$, is at least 
        \[
            \frac{\polylog(n)}{\polylog(s)}\cdot\frac{1}{n^{1+\frac{3}{d-1}}}.
        \]
    \end{proposition}

    Before we prove the proposition, however, we will first need the following definition and lemmas.

    \begin{definition}\label{def:skew2}
        Let $M(\bx)=\max\{0,x_1,x_2,\ldots,x_d\}$ denote the max function. Let $M_1(\bx)$ denote only the portion of the max function where coordinate $x_1$ is biggest, namely, the function that takes value $x_1$ if $x_1$ is the largest of $0,x_2,x_3,\ldots,x_d$, and 0 otherwise. Equivalently, let $q_1(\bx)=x_1\cdot \mathbbm{1}_{[x_1\geq 0]}\cdot \mathbbm{1}_{[x_2\leq 0]}\cdot \mathbbm{1}_{[x_3\leq 0]}\cdot\ldots\cdot \mathbbm{1}_{[x_d\leq 0]}$, namely the function taking value $x_1$ but only when $x_1$ is nonnegative and all the other coordinates are nonpositive; let $S_1(\bx)$ be the ``skew" matrix such that  $S_1\cdot (x_1,x_2,x_3,\ldots,x_d)^T=(x_1,x_2-x_1,x_3-x_1,\ldots,x_d-x_1)^T$, namely subtracting $x_1$ from all the other coordinates; thus $M_1(\bx)=q_1(S_1\cdot \bx))$ since $x_1$ is at least some other coordinate $x_j$ if and only if $x_j-x_1\leq 0$.
    \end{definition}
    
    \begin{lemma}
        The Fourier transform of the skew of a function is the inverse transpose skew of the Fourier transform of the function: for a function $g$, we have $\widehat{g(S_1\cdot \bx)}=\widehat{g}({S_1^{-\top}} \bxi)$, using the standard notation $S_1^{-\top}$ to represent the matrix inverse transpose.
    \end{lemma}

    \begin{proof}
        Standard (linear) change of variables relation for the Fourier transform integral.
    \end{proof}
    
    \begin{lemma}\label{lem:transform-bound}
        Letting $q_1(\bx)=x_1\cdot \mathbbm{1}_{[x_1\geq 0]}\cdot \mathbbm{1}_{[x_2\leq 0]}\cdot \mathbbm{1}_{[x_3\leq 0]}\cdot\ldots\cdot \mathbbm{1}_{[x_d\leq 0]}$ and defining Dawson's integral to be $\daw(x)\coloneqq\exp(-x^2)\int_0^x \exp(t^2)\,dt$, then the Fourier transform of $q_1(\bx) \exp(-\norm{\bx}_2^2/4)$ equals 
        \[
            \p{2-4\xi_1 \daw(\xi_1)-2i\sqrt{\pi}\xi_1 \exp(-\xi_1^2)}\prod_{j=2}^d (\exp(-\xi_j^2)\sqrt{\pi}+2i\cdot \daw(\xi_j)).
        \]    
        And further, this Fourier transform has magnitude at most $2\pi^{\frac{d-1}{2}}$ everywhere, and for inputs $\bxi$ each of whose coordinates is positive and has value at least $\Omega(\log (d))$, the Fourier transform has a component in the (complex) direction $-i^{d-1}$ that is at least $\frac{1}{\xi_1^2}\prod_{j=2}^d\frac{1}{\xi_j}$.
    \end{lemma}
    
    \begin{proof}
        The Fourier transform is a straightforward calculation on each dimension separately, since the function $q_1(\bx) \exp(-\norm{\bx}_2^2/4)$ is separable.
        
        The global magnitude bound simply comes from evaluating the Fourier transform at the origin, since the Fourier transform of a nonnegative real function attains its largest magnitude at the origin.
        
        For the final bound, we take the approximation of Dawson's integral $\daw(x)=\frac{1}{2x}+\frac{1}{4x^3}+\Theta(\frac{1}{x^5})$ for inputs $x$ away from 0. These inverse polynomial terms in $\xi_j$ dominate the inverse exponential $\exp(-\xi_j^2)$ terms, even when $d$ such terms are multiplied together, for $\xi=\Omega(\log d)$. Substituting in this approximation for $\daw(x)$ into our Fourier transform expression and dropping lower-order terms yields $-\frac{1}{\xi_1^2}\prod_{j=2}^d\frac{i}{\xi_j}$, with the next-largest terms from the expansion of Dawson's integral contributing inverse-polynomially farther in the same direction. Thus we conclude the lemma.
    \end{proof}
    \begin{lemma}\label{lem:big-fourier}
        For vector $\bxi$ with all of its coordinates positive and at least $\Omega(d)$, but less than some parameter $b$, the Fourier transform of $\exp(-\norm{\bx}_2^2/4)\max(0,x_1,x_2,\ldots,x_d)$ evaluated at $\bxi$ has magnitude at least $b^{-(d+1)}2^{-\Ocal(d)}$. 
    \end{lemma}
    \begin{proof}
        The Fourier transform of $\exp(-\norm{\bx}_2^2/4)\max\{0,x_1,x_2,\ldots,x_d\}$ can be decomposed into the sum of the contributions of the $d$ separate components of the $\max$ function, which are all symmetric up to relabeling the coordinates. We thus compute the contribution from the first component.
        
        We compute the Fourier transform of $\exp(-\norm{\bx}_2^2/4) M_1(\bx)$ by expressing $M_1=q_1\circ s_1$ from Definition~\ref{def:skew2}, as the composition of a separable function $q_1$ with a volume-preserving affine transformation $s_1$. We make further use of the transformation $s_1$ by breaking the scaling term $\exp(-\norm{\bx}_2^2/4)$ into 2 parts, one of which is a spherical Gaussian even after begin transformed by $s_1$. Explicitly, we have 
        
        \begin{equation}\label{eq:gaussian-skew}\exp(-\norm{\bx}_2^2/4) M_1(\bx)=\exp\p{-\bx^{\top} Q \bx} \exp\p{-||s_1(\bx)||_2^2/(4(d+1))} q_1(s_1(\bx))\end{equation}
        for some symmetric positive semidefinite matrix $Q$ with eigenvalues at most $\frac{1}{4}$. 
        
        Since the Fourier transform of a product equals the convolution of the Fourier transform of the terms, we thus have that the Fourier transform of Equation~\ref{eq:gaussian-skew} equals the convolution of the Fourier transform of the Gaussian $\exp\p{-\bx^{\top} Q \bx}$ with the Fourier transform of the expression $\exp\p{-||s_1(\bx)||_2^2/(4(d+1))} q_1(s_1(\bx))$. Since this expression is an affine transformation of \linebreak $\exp\p{-||\by||_2^2/(4(d+1))} q_1(\by)$, its Fourier transform is the corresponding (inverse transpose) affine transformation of the Fourier transform of $\exp\p{-||\by||_2^2/(4(d+1))} q_1(\by)$. 
        
        We bound this Fourier transform via Lemma~\ref{lem:transform-bound}. Explicitly, let $g(\bxi)$ be the Fourier transform of $\exp\p{-||\by||_2^2/4} q_1(\by)$, which Lemma~\ref{lem:transform-bound} bounds. Then the Fourier transform of \linebreak $\exp\p{-||\by||_2^2/(4(d+1))} q_1(\by)$ is exactly $g(\bxi\sqrt{d+1})(d+1)^\frac{d+1}{2}$, since replacing $\by$ by $\by\sqrt{d+1}$ scales the function $q_1$ by $\sqrt{d+1}$ and thus scales its integral and hence Fourier transform by the $d+1$ power of this, as claimed. Next, transforming the inputs of a function by the affine function $s_1$ transforms its Fourier transform by the transpose of the inverse of the affine function. Thus the Fourier transform of $\exp\p{-||s_1(\bx)||_2^2/(4(d+1))} q_1(s_1(\bx))$ equals $g(\sqrt{d+1}((\sum_j \xi_j),\xi_2,\xi_3,\ldots,\xi_d))\cdot (d+1)^\frac{d+1}{2}$.
        
        We now use the bounds of Lemma~\ref{lem:transform-bound} to characterize $g$. For $\bxi$ with all coordinates positive and at least $\Omega(1)$, the transformed coordinates $\sqrt{d+1}((\sum_j \xi_j),\xi_2,\xi_3,\ldots,\xi_d)$ will all be at least $\Omega(\log (d))$ and thus the lemma applies. Thus we conclude that the Fourier transform has component in the direction $-i^{d-1}$ at least $\frac{1}{(\sum_j \xi_j)^2}\prod_{j=2}^d \frac{1}{\xi_j}$, where the factors of $d+1$ all cancel; by the second part of Lemma~\ref{lem:transform-bound}, this Fourier transform has magnitude at most $\Ocal(d)^{\Ocal(d)}$ everywhere.
        
        Finally, to obtain the overall Fourier transform of 
        $\exp\p{-||\bx||_2^2/4} M_1(\bx)$, it remains to convolve this last expression with the Fourier transform of the remaining term $\exp\p{-\bx^{\top} Q \bx}$; namely, we convolve this last expression with the Gaussian of covariance $Q$, which thus has radius $\leq\frac{1}{4}$ by construction. Since all but $\exp\p{-\Omega(d^2)}$ fraction of the mass of this Gaussian must be within radius $\Ocal(d)$, we thus have that---even after this final convolution---the component of the Fourier transform of $\exp\p{-||\bx||_2^2/4} M_1(\bx)$ in the direction $-i^{d-1}$ must be at least $2^{-\Ocal(d)}\frac{1}{(\sum_j \xi_j)^2}\prod_{j=2}^d \frac{1}{\xi_j}$ provided all coordinates of $\bxi$ are at least $\Omega(d)$. 
        
        Summing this bound over all $d$ components of the maximum function, and then pointing out that the magnitude of a complex number must be at least its component in the direction $-i^{d-1}$ yields our final bound.
    \end{proof}
    
    Using the above lemmas, we now turn to the proof of the proposition.

    \begin{proof}[Proof of \propref{prop:finite_d}]
        Let $b=\omega(d)$. Consider the region in Fourier space where each coordinate $\xi_j$ lies in $[\Omega(d),b]$. This region has volume $\Omega(b^d)$. By Lemma~\ref{lem:big-fourier}, the Fourier transform of the maximum function, weighted by the square root of the pdf of the unit Gaussian, has magnitude at least $b^{-(d+1)}2^{-\Ocal(d)}$; denote this bound by $\ell$.

        However, a ReLU with bounded coefficients has a Fourier transform which is large only on a relatively small volume, which is what gives us the desired contradiction.
        
        We will show that the linear combination of ReLU units cannot closely approximate the max function on our Gaussian via the following ``accounting" scheme: consider the contribution of each neuron separately, and letting $f_k(\bxi)$ be the Fourier transform of the contribution of the $k^\textrm{th}$ neuron, then we give this neuron a ``score" of $\int_{[\Omega(d),b]^d}\min\set{1,\frac{1}{\ell}|f_k(\bxi)|}\,d\bxi$. We will show that the total score over all neurons is at most half the volume of the region, which implies a squared-$L_2$ error of at least $\Omega(b^d \ell^2)=b^{-d-2} 2^{-\Ocal(d)}$.
        
        Consider a neuron with an activation function $\sigma:\mathbb{R}\rightarrow\mathbb{R}$ satisfying Assumption~\ref{asm:poly_bounded}, and weights of magnitude at most some bound $s$. In the context of the neural net, $\sigma$ will be applied as $w \sigma(\bx\cdot \bv)$ where $w$ is a weight of magnitude at most $s$ and $\bv$ is a vector each of whose coordinates has magnitude at most $s$. We decompose the Fourier transform of $w \sigma(\bx\cdot \bv) \exp\p{-||\bx||_2^2/4}$ into the Fourier transform along the direction of $\bv$, and then the Fourier transform in the transverse $d-1$ dimensional space.
        
        Since by Assumption~\ref{asm:poly_bounded}, $\sigma$ is polynomially bounded, the 1-dimensional Fourier transform along direction $\bv$ of this scaled version $w \sigma(\bx\cdot \bv) \exp\p{-||\bx||_2^2/4}$ will be bounded by $\Ocal(s d)^{d+1}$, where the parameters of the big-O expression depend on $\sigma$. We now consider the Fourier transform of along the $d-1$ dimensional space orthogonal to $d$: along any hyperplane orthogonal to $\bv$, we have a $d-1$ dimensional Gaussian $\exp\p{-||\bx||_2^2/4}$ times some number bounded by $\Ocal(s d)^{d+1}$. We thus consider how much ``score" this can contribute.
        
        Namely, for a scaling factor $t=\Ocal(s d)^{d+1}$, we bound $\int_{\mathbb{R}^{d-1}}\min\set{1,\frac{t}{\ell} (2\sqrt{\pi})^{d-1}\exp\p{-||\bxi||_2^2}}\,d\bxi$, where the expression $(2\sqrt{\pi})^{d-1} \exp\p{-||\bxi||_2^2} $ is the $d-1$ dimensional Fourier transform of the Gaussian $\exp\p{-\bx^2/4}$. This integral is straightforward to bound once we convert it to an integral over the radius. We solve for the radius $r$ where the $\min$ function transitions from the first term to the second: $\frac{t}{l}(2\sqrt{\pi})^{d-1}\exp\p{-r^2}\leq 1$ means that $r\geq \sqrt{\log\frac{t}{l}(2\sqrt{\pi})^{d-1}}$. Since the surface area of a radius $r$ ball in $d-1$ dimensions is $\Ocal(1) r^{d-1}$, we bound our integral as
        \[
            \int_{\mathbb{R}^{d-1}}\min\set{1,\frac{t}{\ell} (2\sqrt{\pi})^{d-1}\exp\p{-||\bxi||_2^2}}\,d\bxi\leq \int_0^\infty \Ocal(1) r^{d-1} \min\set{1,\frac{t}{\ell} (2\sqrt{\pi})^{d-1}\exp\p{-r^2}}\,dr
        \]
        And for $\log \frac{t}{\ell} (2\sqrt{\pi})^{d-1}\geq d$ this integral is bounded by $\Ocal(1)\cdot (\log \frac{t}{\ell}(2\sqrt{\pi})^{d-1})^{\frac{d}{2}}$. Substituting in our bound for $t$ yields that, for $\ell\leq \Ocal(s d)^{d+1}$, the contribution to the score per unit in the transverse direction $\bv$ is at most $((d+1)\log \Ocal(s d)+\log\frac{1}{\ell})^{\frac{d}{2}}$.
        
        Since we integrate the score over the hypercube of side length $b$, the projection to direction $\bv$ has length at most $b\sqrt{d}$, and thus the total score of our neural network with $n$ neurons is at most $n b\sqrt{d}((d+1)\log \Ocal(s d)+\log\frac{1}{\ell})^{\frac{d}{2}}$. Substituting in the definition $\ell=b^{-(d+1)}2^{-\Ocal(d)}$ yields $n b\sqrt{d}((d+1)\log \Ocal(s db))^{\frac{d}{2}}$.
        
        As described above, we show this neural network does not closely approximate the $\max$ function by showing that this score is less than half the volume of the region of frequency space under consideration, $\frac{1}{2}(b-\Omega(d))^d$, which is true provided $b\geq n^{\frac{1}{d-1}}\frac{\polylog(s)}{\polylog(n)}$, where we emphasize that the base and exponents of the polylog terms may depend on $d$. Thus our overall $L_2$-squared error bound of $b^{-d-2} 2^{-\Ocal(d)}$ becomes $\frac{\polylog(n)}{\polylog(s)}\frac{1}{n^{1+\frac{3}{d-1}}}$, as claimed.
    \end{proof}
    With \propref{prop:finite_d} at our disposal, we can now prove \thmref{thm:depth2}.

    \begin{proof}[Proof of \thmref{thm:depth2}]
        We wish to find some $c_2>0$ such that
        \[
            \E_{\bx\sim\Ucal\p{\pcc{0,1}^d}}\pcc{\p{\Ncal(\bx) - f_d(\bx)}^2} > \Omega\p{\frac{1}{d^{c_2\cdot\ell}}}.
        \]
        To this end, we first define the sets $A\coloneqq[0,1-1/d]^{d-3}$ and $B=[1-1/d,1]^3$. For any natural $n$, denote the uniform distribution over $[0,1]^n$ by $\Dcal_n$ and compute by repeatedly using the law of total expectation
        \begin{align*}
            &\E_{\bx\sim\Dcal_d}\pcc{\p{\Ncal(\bx)-f_d(\bx)}^2} = \E_{\bx_1\sim\Dcal_{d-3}}\pcc{\E_{\bx_2\sim\Dcal_{3}}\pcc{\p{\Ncal(\bx_1,\bx_2)-f_d(\bx_1,\bx_2)}^2|\bx_1}}\\
            &\hskip 1cm= \E_{\bx_1\sim\Dcal_{d-3}}\pcc{\E_{\bx_2\sim\Dcal_{3}}\pcc{\p{\Ncal(\bx_1,\bx_2)-f_d(\bx_1,\bx_2)}^2|\bx_1\in A,\bx_2\in B}}\cdot\pr\pcc{\bx_1\in A,\bx_2\in B}\\
            &\hskip 1.5cm + \E_{\bx_1\sim\Dcal_{d-3}}\pcc{\E_{\bx_2\sim\Dcal_{3}}\pcc{\p{\Ncal(\bx_1,\bx_2)-f_d(\bx_1,\bx_2)}^2|\bx_1\notin A\text{ or }\bx_2\notin B}}\cdot\pr\pcc{\bx_1\notin A\text{ or }\bx_2\notin B}\\
            &\hskip 1cm \ge \E_{\bx_1\sim\Dcal_{d-3}}\pcc{\E_{\bx_2\sim\Dcal_{3}}\pcc{\p{\Ncal(\bx_1,\bx_2)-f_d(\bx_1,\bx_2)}^2|\bx_1\in A, \bx_2\in B}}\cdot\pr\pcc{\bx_1\in A,\bx_2\in B}\\
            &\hskip 1cm \ge \E_{\bx_1\sim\Dcal_{d-3}}\pcc{\E_{\bx_2\sim\Dcal_{3}}\pcc{\p{\Ncal(\bx_1,\bx_2)-f_d(\bx_1,\bx_2)}^2|\bx_1\in A, \bx_2\in B}}\cdot d^{-3}\exp(-1)\\
            &\hskip 1cm= \E_{\bx_2\sim\Dcal_{3}}\pcc{\p{\Ncal(\bx'_1,\bx_2)-f_d(\bx'_1,\bx_2)}^2|\bx_2\in B}\cdot d^{-3}\exp(-1).
        \end{align*}
        In the above, the last equality holds for some intermediate point $\bx'_1\in A$ whose existence is guaranteed by \lemref{lem:intermediate_value_thm} due to the fact that $g(\bx)=\E_{\bx_2\sim\Dcal_{3}}\pcc{\p{\Ncal(\bx,\bx_2)-f_d(\bx,\bx_2)}^2|\bx_2\in B}$ is continuous on $[0,1]^{d-3}$. Since $\bx_2\mapsto\Ncal(\bx'_1,\bx_2)$ defines a depth 2 $\sigma$ network which we denote by $\Tilde{\Ncal}(\cdot)$, and by using the fact that $f_d(\bx'_1,\bx_2)=f_3(\bx_2)$ for all $\bx'_1\in A$ and $\bx_2\in B$, we can let $\Ncal_{(2)}$ denote the class of depth 2 networks employing a $\sigma$ activation function, and lower bound the above by
        \[
            \inf_{\Tilde{\Ncal}\in\Ncal_{(2)}}\E_{\bx_2\sim \Ucal(B)}\pcc{\p{\Tilde{\Ncal}(\bx_2)-f_3(\bx_2)}^2}\cdot d^{-3}\exp(-1).
        \]
        It now suffices to lower bound the expectation term above by $\Omega(d^{-c'\cdot\ell})$ for some constant $c'>0$. Focusing on the expectation term above and letting $\Tilde{\Ncal}$ denote arbitrary (not necessarily fixed) elements in $\Ncal_{(2)}$, we once more use the law of total expectation repeatedly to obtain
            \begin{align*}
                &\E_{\bx_2\sim \Ucal(B)}\pcc{\p{\Tilde{\Ncal}(\bx_2)-f_3(\bx_2)}^2}\\
                &\hskip 0.2cm =\E_{\Tilde{x}_1\sim\Ucal\p{\pcc{1-\frac1d,1}}}\pcc{\E_{\Tilde{\bx}_2\sim\Ucal\p{\pcc{1-\frac1d,1}^2}}\pcc{\p{\Tilde{\Ncal}(\Tilde{x}_1,\Tilde{\bx}_2)-f_3(\Tilde{x}_1,\Tilde{\bx}_2)}^2|\Tilde{x}_1\in\pcc{1-\frac1d,1}}}\\
                &\hskip 0.2cm \ge0.5\E_{\Tilde{x}_1\sim\Ucal\p{\pcc{1-\frac1d,1}}}\pcc{\E_{\Tilde{\bx}_2\sim\Ucal\p{\pcc{1-\frac1d,1}^2}}\pcc{\p{\Tilde{\Ncal}(\Tilde{x}_1,\Tilde{\bx}_2)-f_3(\Tilde{x}_1,\Tilde{\bx}_2)}^2|\Tilde{x}_1}|\Tilde{x}_1\in\pcc{1-\frac{3}{4d},1-\frac{1}{4d}}}\\
                &\hskip 0.2cm =0.5\E_{\Tilde{\bx}_2\sim\Ucal\p{\pcc{1-\frac1d,1}^2}}\pcc{\p{\Tilde{\Ncal}(x_0,\Tilde{\bx}_2)-f_3(x_0,\Tilde{\bx}_2)}^2},
            \end{align*}
            where the last equality uses \lemref{lem:intermediate_value_thm} to establish the existence of some intermediate point $x_0\in\pcc{1-\frac{3}{4d},1-\frac{1}{4d}}$ satisfying the above. Writing the above expectation term in integral form, we have that it equals
            \[
                \int_{1-\frac1d}^1\int_{1-\frac1d}^1\p{\Tilde{\Ncal}(x_0,\Tilde{\bx}_2)-f_3(x_0,\Tilde{\bx}_2)}^2d^{2}d\Tilde{\bx}_2.
            \]
            By the change of variables $\Tilde{\bx}_2=\by+(x_0,x_0)$, $d\Tilde{\bx}_2=d\by$, the above equals
            \begin{align*}
                &\int_{1-\frac1d-x_0}^{1-x_0}\int_{1-\frac1d-x_0}^{1-x_0}\p{\Tilde{\Ncal}(x_0,\by+(x_0,x_0))-f_3(x_0,\by+(x_0,x_0))}^2d^{2}d\by\\
                =& \int_{1-\frac1d-x_0}^{1-x_0}\int_{1-\frac1d-x_0}^{1-x_0}\p{\Tilde{\Ncal}(x_0,\by+(x_0,x_0))-x_0-f_3(0,\by)}^2d^{2}d\by\\
                =& \int_{1-\frac1d-x_0}^{1-x_0}\int_{1-\frac1d-x_0}^{1-x_0}\p{\Tilde{\Ncal}(\by)-f_3(0,\by)}^2d^{2}d\by\\
                \ge&\int_{-\frac{1}{4d}}^{\frac{1}{4d}}\int_{-\frac{1}{4d}}^{\frac{1}{4d}}\p{\Tilde{\Ncal}(\by)-f_3(0,\by)}^2d^{2}d\by,
            \end{align*}
            where the first equality uses the fact that $f_3(\bx+(c,c,c))=c+f_3(\bx)$ for any vector $\bx$ and real $c$, the second equality follows from the fact that $\Ncal_{(2)}$ is closed under linear transformations of the input and the output, and since we can simulate the fixed input $x_0$ in the first coordinate by an appropriate linear rescaling of the first hidden layer, and the inequality follows from $x_0\in\pcc{1-\frac{3}{4d},1-\frac{1}{4d}}$ which implies that $\pcc{-0.25d^{-1},0.25d^{-1}}\subseteq\pcc{1-1/d-x_0,1-x_0}$. Letting $\gamma>0$ to be determined later, we perform a second linear change of variables $\by=0.5\gamma^{-1}\bz$, $d\by=0.5\gamma^{-2}d\bz$, which entails that the above displayed equation equals
            \begin{align}
                &0.5\int_{-\frac{\gamma}{2d}}^{\frac{\gamma}{2d}}\int_{-\frac{\gamma}{2d}}^{\frac{\gamma}{2d}}\p{\Tilde{\Ncal}\p{0.5\gamma^{-1}\bz}-f_3\p{0,0.5\gamma^{-1}\bz}}^2d^{2}\gamma^{-2}d\bz\nonumber\\            =&0.5\int_{-\frac{\gamma}{2d}}^{\frac{\gamma}{2d}}\int_{-\frac{\gamma}{2d}}^{\frac{\gamma}{2d}}\p{0.5\gamma^{-1}2\gamma\Tilde{\Ncal}\p{0.5\gamma^{-1}\bz}-0.5\gamma^{-1}f_3\p{0,\bz}}^2d^{2}\gamma^{-2}d\bz\nonumber\\
                =&0.25\pi d^{2}\gamma^{-2}\int_{-\frac{\gamma}{2d}}^{\frac{\gamma}{2d}}\int_{-\frac{\gamma}{2d}}^{\frac{\gamma}{2d}}\p{\Tilde{\Ncal}\p{\bz}-f_3\p{0,\bz}}^2\frac{1}{2\pi}d\bz\nonumber\\
                \ge& 0.25\pi d^{2}\gamma^{-2}\int_{\{\bz:\norm{\bz}_2\le 0.5d^{-1}\gamma\}}\p{\Tilde{\Ncal}\p{\bz}-f_3\p{0,\bz}}^2\frac{1}{2\pi}\exp\p{-0.5\norm{\bz}_2^2}d\bz\nonumber\\
                =&0.25\pi d^{2}\gamma^{-2}\left(\E_{\bz\sim\Ncal(\bzero,I_2)}\pcc{\p{\Tilde{\Ncal}\p{\bz}-f_3\p{0,\bz}}^2}\right.\nonumber\nonumber\\ 
                &\hskip 0.2cm\left.- \int_{\{\bz:\norm{\bz}_2\ge 0.5d^{-1}\gamma\}}\p{\Tilde{\Ncal}\p{\bz}
                 -f_3\p{0,\bz}}^2\frac{1}{2\pi}\exp\p{-0.5\norm{\bz}_2^2}d\bz\right).\label{eq:gaussian_dif}
            \end{align}
            In the above, the first equality follows from the fact that $f_3(\alpha\cdot\bx)=\alpha f_3(\bx)$ for all $\alpha>0$ and $\bx\in\reals$, the second equality follows from the fact that $\Ncal_{(2)}$ is closed under linear scaling of its input and output, and the inequality follows from $\{\bz:\norm{\bz}_2\le 0.5d^{-1}\gamma\}\subseteq[-0.5d^{-1}\gamma,0.5d^{-1}\gamma]^2$ and the fact that the maximum of a bivariate standard Gaussian is $\frac{1}{2\pi}$.

            Next, we upper bound the square of the above approximation. We begin with the output of a single neuron:
            \begin{align*}
                \abs{\sigma\p{\inner{\bw_i,\bz} + b_i}} &\le C_{\sigma}\p{1+\abs{\inner{\bw_i,\bz} + b_i}^{\alpha_{\sigma}}} \\ &\le C_{\sigma}\p{1+\abs{\norm{\bw_i}\cdot\norm{\bz} + \abs{b_i}}^{\alpha_{\sigma}}} \le \Ocal\p{\exp\p{\Ocal(d)}\norm{\bz}^{\alpha_{\sigma}}},
            \end{align*}
            where the second inequality follows from Cauchy-Schwartz and the last inequality follows from our assumption on the magnitude of the weights. Using the above, we can upper bound the output of the network by
            \[
                \abs{\Tilde{\Ncal}\p{\bz}} = \abs{\sum_{i=1}^k \sigma\p{\inner{\bw_i,\bz} + b_i} + b_0}
                \le d^{\ell}\cdot\Ocal\p{\exp\p{\Ocal(d)}\norm{\bz}^{\alpha_{\sigma}}},
            \]
            implying
            \[
                \p{\Tilde{\Ncal}\p{\bz}
                 -f_3\p{0,\bz}}^2 \le d^{2\ell}\cdot\Ocal\p{\exp\p{\Ocal(d)}\norm{\bz}^{2\alpha_{\sigma}}}.
            \]
            Substituting $\gamma=\Ocal(d^2)$ in \eqref{eq:gaussian_dif} and using the above, we get a lower bound of
            \[
                0.25\pi d^{-2}\p{\E_{\bz\sim\Ncal(\bzero,I_2)}\pcc{\p{\Tilde{\Ncal}\p{\bz}-f_3\p{0,\bz}}^2} - \Ocal\p{\exp\p{-0.5d^2}}}.
            \]
            Lastly, using \propref{prop:finite_d} to lower bound the expectation term above with the assumed bounds on the parameters, the theorem follows.
        \end{proof}

    \subsection{Proof of \thmref{thm:depth3_lb}}\label{app:depth3_lb}
    
    \begin{proof}
        We begin with reducing the approximation error of $f_d$ over a depth 3 ReLU network to the approximation error of $f_3$ over a depth 2 ReLU network. To this end, we first identify three coordinates in the domain of $\Ncal$ where a certain sub-cube $B$ of dimension 3, and a set $A\subseteq[0,1]^{d-3}$ exist, which satisfy the following properties:
        \begin{enumerate}
            \item \label{item:1}
            $\pr_{\bx\sim\Dcal_{d-3}}\pcc{\bx\in A} \ge 0.1$.
            \item \label{item:2}
            $\pr_{\bx\sim\Dcal_{3}}\pcc{\bx\in B} = d^{-18}$.
            \item \label{item:3}
            $f_d(\bx_1, \bx_2)=f_3(\bx_2)$ for all $\bx_1\in A,\bx_2\in B$.
            \item \label{item:4}
            $A$ can be decomposed into a disjoint partition of at most $2^k$ convex sets $A_1,A_2,\ldots$ and a set of measure zero $\Delta$, such that $A=\bigcup_{j}A_j\cup\Delta$, where for all $j$ and $i\in[k]$, $\pr_{\bx\sim\Dcal_{d-3}}\pcc{{\bx\in A_j}}>0$ and $\sign(n_i(\bx_1,\bx_2))$ is fixed for all $\bx_1\in A_j$ and all $\bx_2\in B$.

        \end{enumerate}
        Before we prove the existence of $A$ and $B$, we shall first show how they imply a reduction to an approximation using depth 2. For any natural $n$, denote the uniform distribution over $[0,1]^n$ by $\Dcal_n$ and compute by repeatedly using the law of total expectation
        \begin{align*}
            &\E_{\bx\sim\Dcal_d}\pcc{\p{\Ncal(\bx)-f_d(\bx)}^2} = \E_{\bx_1\sim\Dcal_{d-3}}\pcc{\E_{\bx_2\sim\Dcal_{3}}\pcc{\p{\Ncal(\bx_1,\bx_2)-f_d(\bx_1,\bx_2)}^2|\bx_1}}\\
            &\hskip 1cm= \E_{\bx_1\sim\Dcal_{d-3}}\pcc{\E_{\bx_2\sim\Dcal_{3}}\pcc{\p{\Ncal(\bx_1,\bx_2)-f_d(\bx_1,\bx_2)}^2|\bx_1\in A,\bx_2\in B}}\cdot\pr\pcc{\bx_1\in A,\bx_2\in B}\\
            &\hskip 1.5cm + \E_{\bx_1\sim\Dcal_{d-3}}\pcc{\E_{\bx_2\sim\Dcal_{3}}\pcc{\p{\Ncal(\bx_1,\bx_2)-f_d(\bx_1,\bx_2)}^2|\bx_1\notin A\text{ or }\bx_2\notin B}}\cdot\pr\pcc{\bx_1\notin A\text{ or }\bx_2\notin B}\\
            &\hskip 1cm\ge \E_{\bx_1\sim\Dcal_{d-3}}\pcc{\E_{\bx_2\sim\Dcal_{3}}\pcc{\p{\Ncal(\bx_1,\bx_2)-f_d(\bx_1,\bx_2)}^2|\bx_1\in A, \bx_2\in B}}\cdot\pr\pcc{\bx_1\in A,\bx_2\in B}\\
            &\hskip 1cm= \sum_j\E_{\bx_1\sim\Dcal_{d-3}}\pcc{\E_{\bx_2\sim\Dcal_{3}}\pcc{\p{\Ncal(\bx_1,\bx_2)-f_d(\bx_1,\bx_2)}^2|\bx_2\in B,\bx_1}|\bx_1\in A_j}\cdot\pr\pcc{\bx_1\in A_j,\bx_2\in B}\\
            &\hskip 1cm= \sum_j\E_{\bx_2\sim\Dcal_{3}}\pcc{\p{\Ncal(\bx'_j,\bx_2)-f_d(\bx'_j,\bx_2)}^2|\bx_2\in B}\cdot\pr\pcc{\bx_1\in A_j,\bx_2\in B}.
        \end{align*}
        In the above, the penultimate equality holds despite the omission of $\Delta$ from the decomposition of $A$ since it is a set of measure zero, and the last equality holds for a set of intermediate points $(\bx'_1,\bx'_2,\ldots)\in\reals^{d-3}$ whose existence is guaranteed by \lemref{lem:intermediate_value_thm} due to \itemref{item:4} and the fact that $g(\bx)=\E_{\bx_2\sim\Dcal_{3}}\pcc{\p{\Ncal(\bx,\bx_2)-f_d(\bx,\bx_2)}^2|\bx_2\in B}$ is continuous on $[0,1]^{d-3}$. Since for any given $j$, $\sign(n_i(\bx'_j,\bx_2))$ is fixed for all $i\in[k]$ and all $\bx_2\in B$, we can collapse the first hidden layer of $\Ncal$,\footnote{More formally, we can obtain $\Ncal_j$ from $\Ncal$ and $\bx'_j$ by choosing some arbitrary $\bx_2\in B$ and considering the sign of $n_i(\bx'_j,\bx_2)$ for each $i\in[k]$. If it is negative we can discard the neuron and set its incoming weight in the second layer to $0$, and if it is positive then we discard the ReLU activation and compose the obtained linear transformation with the linear transformation computed by the corresponding neuron in the second layer. In both cases, the neuron in the first layer is either canceled or is absorbed into the second layer, thus removing the first hidden layer altogether without increasing the width of the network.} obtaining a depth 2 ReLU network $\Ncal_j$ for each $\bx'_j$ such that $\Ncal(\bx'_j,\bx_2)=\Ncal_j(\bx_2)$ for all $\bx_2\in B$. Combining the previous argument with \itemref{item:3}, the above displayed equation is equal to
        \begin{align}
            &\sum_j\E_{\bx_2\sim\Dcal_{3}}\pcc{\p{\Ncal_j(\bx_2)-f_3(\bx_2)}^2|\bx_2\in B}\cdot\pr\pcc{\bx_1\in A_j,\bx_2\in B}\nonumber\\
            &\hskip 0.5cm = \sum_j\E_{\bx_2\sim \Ucal(B)}\pcc{\p{\Ncal_j(\bx_2)-f_3(\bx_2)}^2}\cdot\pr\pcc{\bx_1\in A_j}\cdot\pr\pcc{\bx_2\in B}.\nonumber
        \end{align}
        Letting $\Ncal_{(2)}$ denote the class of depth 2 ReLU networks of width at most $d^2/5$, and letting $b>a\ge0$ such that $B=[a,b]^3$ where $b\coloneqq a+d^{-6}$, we can lower bound the above by
        \begin{align}
            &\sum_j\inf_{\Tilde{\Ncal}\in\Ncal_{(2)}}\E_{\bx_2\sim \Ucal(B)}\pcc{\p{\Tilde{\Ncal}(\bx_2)-f_3(\bx_2)}^2}\cdot\pr\pcc{\bx_1\in A_j}\cdot\pr\pcc{\bx_2\in B}\nonumber\\
            &\hskip 0.5cm= \inf_{\Tilde{\Ncal}\in\Ncal_{(2)}}\E_{\bx_2\sim \Ucal(B)}\pcc{\p{\Tilde{\Ncal}(\bx_2)-f_3(\bx_2)}^2}\cdot\pr\pcc{\bx_1\in A}\cdot\pr\pcc{\bx_2\in B}\nonumber\\
            &\hskip 0.5cm \ge 0.1d^{-18}\cdot\inf_{\Tilde{\Ncal}\in\Ncal_{(2)}}\E_{\bx_2\sim \Ucal(B)}\pcc{\p{\Tilde{\Ncal}(\bx_2)-f_3(\bx_2)}^2},\label{eq:tiny_cube}
        \end{align}
        where the equality is due to \itemref{item:4} and the inequality is due to Items~\ref{item:1} and \ref{item:2}. Applying \lemref{lem:rescaling} and \thmref{thm:depth2} with input dimension $3$ and $\ell=2$, the lower bound follows.


        It now remains to show the existence of $A$ and $B$. Starting with $B$, we first assume w.l.o.g.\ that no neuron in the first hidden layer of $\Ncal$ has an all-zero weight vector. This is justified since if such a neuron exists, it merely outputs a constant as input to the second layer which can be simulated by modifying the bias terms in the second layer, which doesn't increase the width of $\Ncal$. Denote for all $i\in[k]$, $w_i^{\max}=w_{i,j_i}$ where $j_i=\argmax{j\in[d]}\abs{w_{i,j}}$, we define the set
        \[
            P\coloneqq\set{x\in\pcc{1-\frac1d,1}:\abs{x+\frac{b_i}{w_i^{\max}}}\le d^{-3},\hskip 0.3cm\forall\hskip 0.1cm i\in[k]}.
        \]
        Note that by our assumption $k\le\frac{d^2}{5}$, we have that $P$ consists of at most $\frac{d^2}{5}$ connected components where each is of length at most $2d^{-3}$. Therefore, the overall length of $P$ is no more than $\frac{2}{5d}$, and we can thus find an interval $[a,a+d^{-6}]\subseteq [1-1/d]\setminus P$ for some $a\in[1-1/d,1-d^{-6}]$. We can now define
        \[
            B\coloneqq \pcc{a,a+d^{-6}}^3.
        \]
        Note that this immediately entails that $\pr_{\bx\sim\Dcal_3}\pcc{\bx\in B} = d^{-18}$, proving \itemref{item:2}. Continuing to showing the existence of $A$, we first define it formally as the set given by
        \[
            A\coloneqq\pcc{0,1-\frac{1}{d}}^{d-3}\setminus\set{\bx_1\in[0,1]^{d-3}: \exists \bx_2\in B,\hskip 0.1cm i\in[k]\text{ s.t. }n_i(\bx_1,\bx_2)=0}.
        \]
        Note that this immediately implies \itemref{item:3}. To show \itemref{item:4}, we observe that $A$ can be defined as the set difference between a cube and the union of $k$ closed sets, one for each neuron in the first hidden layer of $\Ncal$. More specifically, for $i\in[k]$, suppose that the $i$-th neuron has weights $\bw=(\bw_1,\bw_2)$ and bias $b$ where $\bw_1\in\reals^{d-3}$ and $\bw_2\in\reals^{3}$, and consider the set
        \[
            \Tilde{A}_i \coloneqq \set{\bx_1\in[0,1]^{d-3}: \exists \bx_2\in B\text{ s.t. }n_i(\bx_1,\bx_2)=0}.
        \]
        Observing that
        \[
            \Tilde{A}_i = \bigcup_{\bx_2\in B} \set{\bx_1:\inner{\bw_1,\bx_1} = b-\inner{\bw_2,\bx_2}},
        \]
        we have that $\Tilde{A}_i=\reals^{d-3}$ or $\Tilde{A}_i=\emptyset$ if $\bw_1=\bzero$, depending on whether $b-\inner{\bw_2,\bx_2}$ equals zero for some $\bw_2\in B$ or not. 
        Otherwise, if $\bw_1\neq\bzero$, we have that $\Tilde{A}_i$ can be represented as a union of parallel affine subspaces, the boundary of which is given by $\inner{\bw_1,\bx_1}=\min_{\bx_2\in B}b-\inner{\bw_2,\bx_2}$ and $\inner{\bw_1,\bx_1}=\max_{\bx_2\in B}b-\inner{\bw_2,\bx_2}$, where both the minimum and maximum are defined since $B$ is compact and $\bx_2\mapsto b-\inner{\bw_2,\bx_2}$ is continuous. Moreover, by continuity we also obtain that $\Tilde{A}_i$ is connected. We thus have in either case that we can represent $\overline{\Tilde{A}_i}=\Tilde{A}_{i,1}\cup \Tilde{A}_{i,2}$ for some disjoint and convex sets $\Tilde{A}_{i,1},\Tilde{A}_{i,2}\subseteq\reals^{d-3}$. We now compute
        \begin{align*}
            A&=\pcc{0,1-\frac1d}^{d-3}\setminus\p{\bigcup_{i\in[k]}\Tilde{A}_i} = \pcc{0,1-\frac1d}^{d-3}\hskip 0.05cm\bigcap\hskip 0.1cm
            \overline{\bigcup_{i\in[k]}\Tilde{A}_i}
            = \pcc{0,1-\frac1d}^{d-3}\hskip 0.05cm\bigcap\hskip 0.1cm\p{\bigcap_{i\in[k]}\overline{\Tilde{A}_i}}\\
            &= \pcc{0,1-\frac1d}^{d-3}\hskip 0.05cm\bigcap\hskip 0.1cm \p{\bigcap_{i\in[k]}\p{\Tilde{A}_{i,1}\cup \Tilde{A}_{i,2}}} = \bigcup_{j_1,\ldots,j_k\in\{1,2\}}\p{\p{\bigcap_{i\in[k]}\Tilde{A}_{i,j_i}}\bigcap \pcc{0,1-\frac1d}^{d-3}}.
        \end{align*}
        Namely, $A$ is a union of at most $2^k$ disjoint and convex sets (since the intersection of convex sets is also convex). For $j_1,\ldots,j_k\in\{1,2\}$, denote
        \[
            A_{j_1,\ldots,j_k}\coloneqq\p{\bigcap_{i\in[k]}\Tilde{A}_{i,j_i}}\bigcap \pcc{0,1-\frac1d}^{d-3}.
        \]
        By defining
        \[
            \Delta\coloneqq \set{\bigcup_{j_1,\ldots,j_k\in\{1,2\}}A_{j_1,\ldots,j_k} : \pr_{\bx\sim\Dcal_{d-3}}\pcc{\bx\in A_{j_1,\ldots,j_k}}=0},
        \]
        we can partition $A$ into $\Delta$ and $A\setminus\Delta$ where each connected component in $A\setminus\Delta$ is not a measure zero set. Finally, we have that $\sign(n_i(\bx_1,\bx_2))$ is fixed on each such convex component. This holds true since if otherwise, by contradiction, it holds that some neuron satisfies $n_i(\bx_1,\bx_2)\ge0$ and $n_i(\bx'_1,\bx'_2)<0$ for some $(\bx_1,\bx_2),(\bx'_1,\bx'_2)\in A_j\times B$ and some $j\in[2^k]$. Consider the path $p:[0,1]\to\reals^d$ given by
        \[
            p(\lambda) \coloneqq \lambda(\bx_1,\bx_2) + (1-\lambda)(\bx'_1,\bx'_2).
        \]
        Since $A_j$ and $B$ are convex, so is $A_j\times B$, and we have $p(\lambda)\in A_j\times B$ for all $\lambda\in[0,1]$. Since $n_i(p(\lambda))$ is continuous in $\lambda$, by the intermediate value theorem, we can find some $\lambda_0\in[0,1]$ such that $n_i(p(\lambda_0))=0$, which contradicts the definition of $A$.
        
        To show \itemref{item:1}, we first show that with high probability over drawing $\bx_1$, it holds that $n_i(\bx_1,\bx_2)\neq 0$ for all $\bx_2\in B$ and $i\in[k]$. Suppose that $\bx_1\sim\Ucal\p{[0,1-1/d]^{d-3}}$. We now construct the following graph $G$: $G$ has $d$ vertices, one for each coordinate of the input dimension, and the set of edges is determined according to the values of the weights in the first hidden layer of $\Ncal$. Specifically, there's an edge between two vertices $j_1,j_2\in[d]$ if and only if there exists no neuron $m\in[k]$ such that $|w_{m,j_1}|>|w_{m,\ell}|$ for all $\ell\in[k]\setminus\{j_2\}$ and $|w_{m,j_2}|>|w_{m,\ell}|$ for all $\ell\in[k]\setminus\{j_1\}$. In words, there's no edge between vertices $j_1$ and $j_2$ if and only if there exists a neuron in which the coordinates $j_1,j_2$ have the strictly largest weights in the neuron in absolute value. Since $k\le d^2/5$, we have that $G$ must contain at least $\binom{d}{2}-\frac{d^2}{5}$ edges, which is strictly greater than $\frac{d^2}{4}$ for sufficiently large $d$, and thus by Mantel's theorem (\thmref{thm:mantel}) we have that $G$ must contain a triangle. Consider this triangle in $G$, and assume w.l.o.g.\ that it is formed on the last three coordinates. This means by the definition of $G$ that at least one of the two largest coordinates in each neuron in the hidden layer of $\Ncal$ have an index $j\le d-3$. Fix some neuron, and let $\bw=(w_1,\ldots,w_d)$ and $b$ denote the weights and bias, respectively, of the neuron. Further assume w.l.o.g.\ that $|w_1|\ge |w_2|\ge\ldots\ge|w_{d-3}|$ and that $|w_{d-2}|\le |w_{d-1}|\le |w_d|$. We now perform a case analysis depending on the ratio between the two largest coordinates in the weights of the neuron.
        \begin{itemize}
            \item 
            Suppose that $|w_1|\le |w_{d}|d^{-4}$. Note that this also entails $|w_d|\ge|w_1|d^4>|w_1|$. Namely, $w_d$ has the largest magnitude in absolute value among the weights of the neuron. By our construction of $B$, we have
            \[
                \abs{x+\frac{b}{w_d}} > d^{-3},\hskip 0.3cm \forall \hskip 0.1cm x\in\pcc{a,a+d^{-6}}.
            \]
            Simple algebra and the above imply that
            \begin{equation}\label{eq:notinrange}
                w_dx+b\notin\pcc{-|w_d|d^{-3},|w_d|d^{-3}},\hskip 0.3cm \forall \hskip 0.1cm x\in\pcc{a,a+d^{-6}}.
            \end{equation}
            We now compute
            \[
                \abs{\sum_{j=1}^{d-1}w_jx_j} \le \sum_{j=1}^{d-1}|w_j|\le \sum_{j=1}^{d-1}|w_1| < |w_d|d^{-3}.
            \]
            In the above, the first inequality is H\"older's inequality, the second inequality is due to the fact that $|w_1|$ is the second largest in absolute value among the weights of the neuron, and the last inequality is by our assumption $|w_1|\le |w_{d}|d^{-4}$. Adding the above displayed inequality with \eqref{eq:notinrange} we obtain
            \begin{equation}\label{eq:wp1}
                \sum_{j=1}^dw_jx_j+b\neq0,\hskip 0.3cm \forall \hskip 0.1cm \bx_1\in\pcc{0,1-1/d}^{d-3},\hskip 0.3cm \forall \hskip 0.1cm \bx_2\in B.  
            \end{equation}
                
            \item
            Suppose that $|w_1|> |w_{d}|d^{-4}$. Then in such a case, we cannot guarantee that \eqref{eq:wp1} holds with probability $1$. However, we can show that the randomness over drawing $x_1$ induces sufficient variance and therefore it holds with high probability. 
            Define the random variables $X\coloneqq\sum_{j=1}^{d-3}w_jx_j+b$, $\Tilde{X}\coloneqq\sum_{j=2}^{d-3}w_jx_j+b$, where the randomness is taken over drawing $\bx_1\sim\Ucal([0,1-1/d]^{d-3})$, and let $I\subseteq\reals$ be any interval of length $3|w_d|d^{-6}$. We compute
            \begin{align}
                \pr\pcc{X\in I|\Tilde{X}=x} &= \pr\pcc{w_1x_1+x\in I|\Tilde{X}=x} \le \int_{\reals}\frac{1-1/d}{|w_1|}\one{t+x\in I}dt\nonumber\\
                &\le \frac{3|w_d|d^{-6}}{|w_1|} < 3d^{-2},\label{eq:small_interval}
            \end{align}
            where we used the fact that the density of the random variable $w_1x_1$ is $\frac{1-1/d}{|w_1|}$ in its support, and our assumption which implies that $|w_d|<|w_1|d^4$. Next, compute using the law of total probability to obtain
            \[
                \pr\pcc{X\in I} = \int_{\reals}\pr\pcc{X\in I|\Tilde{X}=x}p_{\Tilde{X}}(x)dx < \int_{\reals}3d^{-2}p_{\Tilde{X}}(x)dx = 3d^{-2},
            \]
            where the inequality follows from \eqref{eq:small_interval}. Since 
            \[
                \max_{\bx_2\in B}\sum_{i=d-2}^dw_ix_i-\min_{\bx_2\in B}\sum_{i=d-2}^dw_ix_i \le 3|w_d|d^{-6},
            \]
            we have that
            \begin{equation}\label{eq:whp}
                \pr_{\bx_1\sim\Ucal([0,1-1/d]^{d-3})}\pcc{\sum_{j=1}^dw_jx_j+b\neq0} \ge 1-3d^{-2}, \hskip 0.3cm \forall \hskip 0.1cm \bx_2\in B.
            \end{equation}
        \end{itemize}
        Having shown that in both cases \eqref{eq:whp} holds, we proceed by taking a union bound over all the $k\le d^2/5$ neurons in the first hidden layer of $\Ncal$, obtaining
        \[
            \pr_{\bx_1\sim\Ucal([0,1-1/d]^{d-3})}\pcc{\sum_{j=1}^dw_{i,j}x_j+b\neq0} \ge \frac{2}{5}, \hskip 0.3cm \forall \hskip 0.1cm i\in [k], \hskip 0.3cm \forall \hskip 0.1cm \bx_2\in B.
        \]
        We now observe that
        \begin{align*}
            &\pr_{\bx_1\sim\Dcal_{d-3}}\pcc{\bx_1\in A} \\
            &\hskip 1cm= \pr_{\bx_1\sim\Dcal_{d-3}}\pcc{\bx_1\in A|\bx_1\in[0,1-1/d]^{d-3}}\cdot\pr_{\bx_1\sim\Dcal_{d-3}}\pcc{|\bx_1\in[0,1-1/d]^{d-3}}\\
            &\hskip 1cm \ge \frac{2}{5}\p{1-\frac{1}{d}}^{d-3} \ge \frac{2}{5}\exp(-1) \ge 0.1,
        \end{align*}
        which thus proves \itemref{item:1}, and completes the proof of the theorem. 
    \end{proof}

    \subsection{Proof of \thmref{thm:size_d_lb}}\label{app:size_d_lb_proof}

    \begin{proof}
        Consider the matrix of first hidden layer weights $W\in\reals^{k\times d}$. Since $k\le d-1$ by our assumption, we have that $\dim(\ker(W))\ge1$. Fix some vector $\bv=(v_1,\ldots,v_d)\in\ker(W)$ such that $\norm{\bv}_2=1$ and assume w.l.o.g.\ $\norm{\bv}_{\infty}=v_1$. Denote $X\coloneqq[0,1]^d$, we now consider the triangular matrix and vector
        \[
            P\coloneqq\p{
                \begin{matrix}
                    \frac1dv_1 & 0 & 0 & \cdots & 0\\
                    \frac1dv_2 & 1-\frac2d & 0 & \cdots & 0\\
                    \frac1dv_3 & 0 & 1-\frac2d & \cdots & 0\\
                    \vdots & \vdots & \vdots & \ddots & \vdots\\
                    \frac1dv_d & 0 & 0 & \cdots & 1-\frac2d\\
                \end{matrix}
            },
            \hskip 0.5cm \bb\coloneqq\p{
                \begin{matrix}
                    1-\frac1d\\
                    \frac1d\\
                    \frac1d\\
                    \vdots\\
                    \frac1d
                \end{matrix}
            },
        \]
        and the set defined by
        \[
            \Pcal \coloneqq\set{P\bx+\bb:\bx\in X}.
        \]
        By its definition, $\Pcal$ is a parallelotope satisfying $\Pcal\subseteq X$. Moreover, we have
        \begin{equation}\label{eq:max_over_para}
            f_d(\bu)=u_1, \hskip 0.3cm \forall\hskip 0.1cm \bu=(u_1,\ldots,u_d)\in\Pcal.
        \end{equation}
        The above holds true since for all $\bu\in\Pcal$ there exists some $\bx=(x_1,\ldots,x_d)\in X$ such that $u_i=1-\frac1d+\frac1d v_ix_1$ for all $i\ge2$ and $u_1=1-\frac1d+\frac1dv_1$, and thus by our assumption that $\norm{\bv}_{\infty}=v_1$ we have
        \[
            u_1=1-\frac1d+\frac1dv_1x_1\ge 1-\frac1d+\frac1d v_ix_1 = u_i.
        \]
        Using the change of variables $\bu=P\bx+\bb$, $d\bu=\abs{\det\p{P}}d\bx$; the fact that $\Pcal\subseteq X$; and the fact that the squared loss is non-negative, we have
        \begin{align}
            \E_{\bu\sim\Ucal\p{X}}\pcc{\p{\Ncal(\bu)-f_d(\bu)}^2} &\ge \int_{\Pcal}\p{\Ncal(\bu)-f_d(\bu)}^2d\bu\nonumber\\
            &=\int_X\p{\Ncal(P\bx+\bb)-f_d(P\bx+\bb)}^2\abs{\det\p{P}}d\bx.\label{eq:change_of_variables}
        \end{align}
        Letting $\be_i$ denote the standard unit vector with coordinate $e_i=1$, we get from $P\bx=\frac1d\bv x_1+\sum_{i=2}^d\p{1-\frac2d}x_i\be_i$ and $\bv\in\ker(W)$ that we can write $\Ncal(P\bx+\bb)=c(x_2,\ldots,x_d)$ for some function $c:\reals^{d-1}\to\reals$. Since $P$ is triangular, we have $\abs{\det(P)}=\frac1d\p{1-\frac2d}^{d-1}v_1 \ge \frac{1}{10d}v_1$. Moreover, since $\norm{\bv}_{\infty}=v_1$ and $\norm{\bv}_2=1$, we have that $v_1\ge d^{-0.5}$ and we can further lower bound the above to obtain $\abs{\det(P)}\ge\frac{1}{10d^{1.5}}$. Plugging the above and \eqref{eq:max_over_para} back in \eqref{eq:change_of_variables}, we obtain
        \begin{align}
            &\E_{\bx\sim\Ucal\p{X}}\pcc{\p{\Ncal(\bx)-f_d(\bx)}^2} \ge \frac{1}{10d^{1.5}}\int_X\p{c(x_2,\ldots,x_d) - \p{1-\frac1d+\frac1dv_1x_1}}^2d\bx\nonumber\\
            &\hskip 1.5cm= \frac{1}{10d^{1.5}}\int_{x_d}\ldots\int_{x_2}\int_{x_1}\p{c(x_2,\ldots,x_d) - \p{1-\frac1d+\frac1dv_1x_1}}^2dx_1dx_2\ldots dx_d.\label{eq:multivariate_integral}
        \end{align}
        It is easy to verify that the optimal constant approximation for the linear function $1-\frac1d+\frac1dv_1x_1$ is $1-\frac1d+\frac1{2d}v_1$, in which case the optimal $L_2$ approximation error is
        \[
            \int_0^1\p{1-\frac1d+\frac1{2d}v_1 - \p{1-\frac1d+\frac1dv_1x}}^2dx = \frac{v_1^2}{d^2}\int_0^1\p{\frac12-x}^2dx = \frac{v_1^2}{12d^2}.
        \]
        Plugging the above back in \eqref{eq:multivariate_integral} and using the fact that $v_1\ge d^{-0.5}$ again, we obtain
        \[
            \E_{\bx\sim\Ucal\p{X}}\pcc{\p{\Ncal(\bx)-f_d(\bx)}^2} \ge \frac{1}{10d^{1.5}}\int_{x_d}\ldots\int_{x_2}\frac{v_1^2}{12d^2} dx_2\ldots dx_d \ge \frac{1}{120d^{4.5}},
        \]
        concluding the proof of the lemma.
    \end{proof}

    \section{Technical lemmas}
    
    The following theorem is a well-known result in graph theory, which we state here for the sake of completeness.

    \begin{theorem}[Mantel's theorem]\label{thm:mantel}
        Let $G$ be a graph with $d$ vertices and more than $d^2/4$ edges. Then $G$ contains a triangle.
    \end{theorem}

    \begin{lemma}\label{lem:add_mul_separated}
        Let $\bx=(x_1,\ldots,x_d)$ and $M,\delta>0$ such that $\abs{x_i-x_j}>\delta$ and $\abs{x_j}\le M$ for all $i$ and $j$. Then $\bx\in S_{\delta/M}$.
    \end{lemma}

    \begin{proof}
        Assuming $x_j\neq 0$, we have that
        \[
            \frac{x_i}{x_j} = 1+\frac{x_i-x_j}{x_j} \notin \pcc{1-\frac{\delta}{M},1+\frac{\delta}{M}} \iff \abs{\frac{x_i-x_j}{x_j}}>\frac{\delta}{M}. 
        \]
        Thus the lemma follows from
        \[
            \abs{\frac{x_i-x_j}{x_j}} > \frac{\delta}{\abs{x_j}} \ge \frac{\delta}{M},
        \]
        where the first inequality is by the assumption $\abs{x_i-x_j}>\delta$ and the second inequality is by the assumption $\abs{x_j}\le M$ which implies $1/\abs{x_j}\ge 1/M$.
    \end{proof}
    
    \begin{lemma}\label{lem:prod}
        \[
            \prod_{i=1}^{\infty}\p{1+\frac{2}{i^3}}^2\le 20.
        \]
    \end{lemma}

    \begin{proof}
        Compute
        \begin{align*}
            \prod_{i=1}^{\infty}\p{1+\frac{2}{i^3}}^2 &= 9\p{\frac54}^2\prod_{i=3}^{\infty}\p{1+\frac{2}{i^3}}^2 = \frac{225}{16}\exp\p{2\sum_{i=3}^{\infty}\ln\p{1+\frac{2}{i^3}}} \\
            &\le \frac{225}{16}\exp\p{2\sum_{i=3}^{\infty}\frac{2}{i^3}} = \frac{225}{16}\exp\p{4\p{\zeta(3)-\frac{9}{8}}}\\
            &\le \frac{225}{16}\exp\p{4\p{1.21-\frac{9}{8}}},
        \end{align*}
        where the first inequality follows from the inequality $\ln(1+x)<x$ for all $x>0$, and the second inequality is a known bound $\zeta(3)\le1.21$ where $\zeta(\cdot)$ is Riemann's zeta function. Evaluating the above expression, the lemma follows.
    \end{proof}

    \begin{lemma}\label{lem:technical_inequality}
        For all natural $d\ge58$ and $1\le k\le\lceil\log(\log(d)+1)\rceil$, we have
        \[
            1\le\frac{2d^{1-\beta(k+1)}}{(k+1)^3}.
        \]
    \end{lemma}

    \begin{proof}
        We first verify the lemma for $k=1$. We have
        \[
            \frac{2d^{1-\beta(2)}}{8}=\frac14d^{\frac{2}{3}} \ge \frac14\sqrt{d} \ge \frac14\sqrt{100} \ge1.
        \]
        Next, assume $k\ge2$ and compute
        \[
            \frac{2d^{1-\beta(k+1)}}{(k+1)^3} \ge \frac{2d^{\frac67}}{(k+1)^3} \ge \frac{2d^{\frac67}}{(\lceil\log(\log(d)+1)\rceil+1)^3}.
        \]
        It thus suffices to prove that
        \[
            \lceil\log(\log(d)+1)\rceil+1 \le 2^{\frac13}d^{\frac{6}{21}}.
        \]
        It is easy to see that this inequality holds for $d=58$ (using any symbolic computation package), and since the left hand side is constant for all $d\in[58,128]$ whereas the right hand side is increasing, the inequality also holds for all $d\le128$. By the same reasoning, we observe that the left hand side is constant on any interval of the form $\left(2^{2^n-1},2^{2^{n+1}-1}\right]$ for integer $n\ge3$, and takes the value of $n+2$. In contrast, the right hand side is lower bounded by $2^{\frac13}\p{2^{2^n-1}}^{\frac{6}{21}}$ on each such interval. It is thus sufficient to prove that
        \[
            n+2\le 2^{\frac13}\p{2^{2^n-1}}^{\frac{6}{21}}
        \]
        for all integer $n\ge3$. We shall show this using induction. The base case can be easily verified for $n=3$. Assuming the induction hypothesis for $n$, we compute
        \begin{align*}
            2^{\frac13}\p{2^{2^{n+1}-1}}^{\frac{6}{21}} &= 2^{\frac13}\p{2^{2^{n}-1}}^{\frac{6}{21}} \cdot \frac{\p{2^{2^{n+1}-1}}^{\frac{6}{21}}}{\p{2^{2^{n}-1}}^{\frac{6}{21}}} \ge (n+2)\p{2^{2^n}}^{\frac{6}{21}}\\
            &\ge(n+2)\p{2^8}^{\frac{6}{21}} \ge 2n+4\ge n+3,
        \end{align*}
        where the first inequality follows from the induction hypothesis, and the second inequality follows from $n\ge 3$.
    \end{proof}

    \begin{lemma}\label{lem:intermediate_value_thm}
        Let $\mu$ denote the $d$-dimensional Lebesgue measure, Let $D\subseteq\reals^d$ be compact, and suppose that $g:D\to\reals$ is continuous and that $\Omega\subseteq[0,1]^d$ is a convex set satisfying $\mu(\Omega)>0$. Then there exists some $\bx_0\in\Omega$ such that
        \[
            \E_{\bx\sim\Ucal\p{D}}\pcc{g(\bx)|\bx\in \Omega} = g(\bx_0).
        \]
    \end{lemma}

    \begin{proof}
        Since $\Omega=(\partial\Omega\cap\Omega)\cup\interior(\Omega)$ is a disjoint union and since the boundary of a convex set in $\reals^d$ has measure zero \citep{lang1986note}, we have
        \[
            \E_{\bx\sim\Ucal\p{D}}\pcc{g(\bx)|\bx\in \Omega} = \frac{1}{\mu(\Omega)}\int_{\Omega}g(\bx)d\bx = \frac{1}{\mu(\Omega)}\int_{\interior(\Omega)}g(\bx)d\bx.
        \]
        Due to the above, we may assume w.l.o.g.\ that $\Omega=\interior(\Omega)$ is thus open and Lebesgue measurable. If $g$ is constant on $\Omega$ then the lemma holds true for all $\bx_0\in\Omega$. Suppose that $g$ is not constant on $\Omega$, then due to being continuous on the compact domain $D\supseteq\Omega$, it is bounded on $\Omega$, and there exist $\bx_1,\bx_2\in\Omega$ such that $g(\bx_1)<g(\bx_2)$. Denote
        \[
            -\infty<m\coloneqq\inf_{\bx\in\Omega} g(\bx)\le g(\bx_1)<g(\bx_2)\le\sup_{\bx\in\Omega} g(\bx) \eqqcolon M<\infty.
        \]
        In particular, by the continuity of $g$ and since $\Omega$ is open, there exists some open neighborhood $U\subseteq\Omega$ containing $\bx_2$ and satisfying $\mu(U)>0$ such that $g(\bx')\ge \frac{g(\bx_1)+g(\bx_2)}{2}>g(\bx_1)\ge m$ for all $\bx'\in U$. We now have
        \begin{align*}
            \frac{1}{\mu(\Omega)} \int_{\Omega}g(\bx)d\bx &= \frac{1}{\mu(\Omega)} \p{\int_{\Omega\setminus U}g(\bx)d\bx + \int_{U}g(\bx)d\bx}\\
            &\ge \frac{1}{\mu(\Omega)} \p{m\cdot\mu\p{\Omega\setminus U} + \int_{U}g(\bx)d\bx}\\
            &> \frac{1}{\mu(\Omega)} \p{m\cdot\mu\p{\Omega\setminus U} + m\cdot\mu(U)} = m.
        \end{align*}
        An analogous argument shows that $\frac{1}{\mu(\Omega)} \int_{\Omega}g(\bx)d\bx<M$, and we thus deduce that
        \begin{equation}\label{eq:intermediate_value}
            m<\frac{1}{\mu(\Omega)} \int_{\Omega}g(\bx)d\bx<M.
        \end{equation}
        Let $\set{\ba_n}_{n=1}^{\infty},\set{\bb_n}_{n=1}^{\infty}\subseteq\Omega$ such that $\lim_{n\to\infty}g(\ba_n)=m$ and $\lim_{n\to\infty}g(\bb_n)=M$. Then for any $\varepsilon>0$, there exists $n_0$ such that $g(\ba_{n_0})\le m+\varepsilon$ and $g(\bb_{n_0})\ge M-\varepsilon$. Consider the path $p:[0,1]\to\reals^d$ given by
        \[
            p(\lambda) \coloneqq \lambda\ba_{n_0} + (1-\lambda)\bb_{n_0}.
        \]
        From the convexity of $\Omega$ we have that $p(\lambda)\in\Omega$ for all $\lambda\in[0,1]$, and since $g(p(\lambda))$ is continuous in $\lambda$, for all $y\in[m+\varepsilon,M-\varepsilon]$ there exists some $\lambda\in[0,1]$ such that
        \[
            g(p(\lambda)) = y.
        \]
        In particular, using \eqref{eq:intermediate_value}, we can choose $\varepsilon>0$ sufficiently small such that
        \[
            \frac{1}{\mu(\Omega)} \int_{\Omega}g(\bx)d\bx \in [m+\varepsilon,M-\varepsilon],
        \]
        and find some $\lambda_0$ satisfying
        \[
            \frac{1}{\mu(\Omega)} \int_{\Omega}g(\bx)d\bx = g(p(\lambda_0)).
        \]
        Letting $\bx_0\coloneqq p(\lambda_0)\in\Omega$ gives the desired result.
    \end{proof}

    \begin{lemma}\label{lem:rescaling}
        Let $\Ncal_{k,\ell}$ denote the class of width $k$ and depth $\ell$ neural networks with an arbitrary activation function. Then there exists some $\varepsilon>0$ such that
        \[
            \inf_{\Ncal\in\Ncal_{k,\ell}}\E_{\bx\sim\Ucal\p{\pcc{0,1}^d}}\pcc{\p{\Ncal(\bx)-f_d(\bx)}^2} \ge \varepsilon,
        \]
        if and only if for all $R>0$ and all $a\in\reals$, we have
        \[
            \inf_{\Ncal\in\Ncal_{k,\ell}}\E_{\bx\sim\Ucal\p{\pcc{a,a+R}^d}}\pcc{\p{\Ncal(\bx)-f_d(\bx)}^2} \ge R^2\varepsilon.
        \]
    \end{lemma}

    \begin{proof}
        The fact that the latter implies the former is immediate by substituting $R=1$ and $a=0$. For the reverse implication, observing that $f_d(\bx)=R\cdot f_d\p{\frac{1}{R}\bx}$ for all $\bx\in\reals^d$, we have by \citep[Theorem~9]{safran2019depth} that
        \[
            \inf_{\Ncal\in\Ncal_{k,\ell}}\E_{\bx\sim\Ucal\p{\pcc{0,1}^d}}\pcc{\p{\Ncal(\bx)-f_d(\bx)}^2} \ge \varepsilon
        \]
        implies
        \[
            \inf_{\Ncal\in\Ncal_{k,\ell}}\E_{\bx\sim\Ucal\p{\pcc{0,R}^d}}\pcc{\p{\Ncal(\bx)-f_d(\bx)}^2} \ge R^2\varepsilon,
        \]
        for all $R>0$. Writing the above expectation in integral form and performing the change of variables $\bx=\by+(a,\ldots,a)$, $d\bx=d\by$, the lemma follows.
    \end{proof}

\end{document}